\newif\ifsiam
\siamfalse
\ifsiam
\documentclass[review,onefignum,onetabnum]{siamonline190516}
\usepackage{enumitem}
\setlist[enumerate]{leftmargin=.5in}
\setlist[itemize]{leftmargin=.5in}

\usepackage{chngcntr}

\counterwithout{algorithm}{section}

\usepackage{crossreftools}
\else
\documentclass[11pt,letterpaper]{article}
\usepackage{amsmath,amsthm}
\usepackage[hidelinks]{hyperref}
\usepackage[capitalize]{cleveref}
\usepackage{algorithm}
\usepackage{fullpage}
\fi

\ifsiam
\newsiamthm{assumption}{Assumption}
\newsiamthm{claim}{Claim}
\newsiamthm{fact}{Fact}
\theoremstyle{nonumberplain}
\theoremheaderfont{\normalfont\itshape}
\theorembodyfont{\normalfont}
\theoremseparator{.}
\theoremsymbol{\proofbox}
\newtheorem{proofidea}{Proof idea}

\else
\newtheorem{theorem}{Theorem}

\newtheorem{definition}[theorem]{Definition}
\newtheorem{lemma}[theorem]{Lemma}
\newtheorem{corollary}[theorem]{Corollary}

\newtheorem{assumption}[theorem]{Assumption}

\newtheorem{fact}[theorem]{Fact}

\newenvironment{proofidea}{\noindent{\textit{Proof idea.}}}{\hfill$\square$\medskip}
\fi

\usepackage[utf8]{inputenc}
\usepackage{amsfonts,bbm,amsopn}
\usepackage{mathtools}
\usepackage{microtype}
\usepackage{algpseudocode}
\usepackage{color}
\usepackage{tikz, graphicx}
\usepackage[caption=false]{subfig}
\numberwithin{theorem}{section}
\numberwithin{equation}{section}

\crefname{section}{Section}{Section}
\crefname{subsection}{Subsection}{Subsection}
\crefname{algocf}{alg.}{algs.}
\Crefname{algocf}{Algorithm}{Algorithms}
\crefname{algocfline}{alg.}{algs.}
\Crefname{algocfline}{Algorithm}{Algorithms}
\Crefname{equation}{}{}
\Crefname{ALC@unique}{Step}{Steps}

\DeclareMathOperator{\Real}{\mathbb{R}}

\DeclareMathOperator{\NatureNumber}{\mathbb{N}}
\newcommand{\abs}[1]{\lvert #1\rvert}
\newcommand{\Abs}[1]{\bigl\lvert #1 \bigr\rvert}
\newcommand{\bigabs}[1]{\bigl\lvert #1 \bigr\rvert}
\newcommand{\biggabs}[1]{\biggl\lvert #1 \biggr\rvert}
\newcommand{\Biggabs}[1]{\Biggl\lvert #1 \Biggr\rvert}
\newcommand{\Bigabs}[1]{\Bigl\lvert #1 \Bigr\rvert}

\newcommand{\inner}[2]{\langle #1, #2\rangle}
\newcommand{\norm}[1]{\lVert #1 \rVert}

\newcommand{\biggnorm}[1]{\biggl\lVert #1 \biggr\rVert}
\newcommand{\snorm}[1]{\lVert #1 \rVert}
\newcommand{\enorm}[1]{{\snorm{#1}}_2}
\DeclareMathOperator{\spanop}{span}
\DeclareMathOperator{\rank}{rank}
\DeclareMathOperator{\vecop}{vec}

\DeclareMathOperator{\diag}{diag}
\DeclareMathOperator{\prob}{\mathbb{P}}
\DeclareMathOperator{\pr}{\mathbb{P}}
\DeclareMathOperator{\expectation}{\mathbb{E}}

\DeclareMathOperator{\var}{Var}
\DeclareMathOperator{\cov}{cov}
\DeclareMathOperator{\dist}{dist}
\DeclareMathOperator{\event}{\mathcal{E}}
\DeclareMathOperator{\sphere}{\mathcal{S}}
\DeclareMathOperator{\proj}{proj}

\newcommand{\RR}{\mathbb{R}}
\newcommand{\eps}{\varepsilon}
\DeclareMathOperator{\poly}{poly}
\newcommand{\krank}{\operatorname{K-rank}}
\newcommand{\rkrank}[2]{\operatorname{K-rank}_{#1}(#2)}

\newcommand{\fnorm}[1]{{\norm{#1}}_F}
\newcommand{\sfnorm}[1]{{\snorm{#1}}_F}
\newcommand{\suchthat}{\mathrel{:}}
\newcommand{\RN}[1]{\textup{\uppercase\expandafter{\romannumeral#1}}}
\newcommand{\hfrac}[2]{{#1}/{#2}}
\newcommand{\giventhat}{\mid}
\newcommand{\ev}[1]{\mathcal{#1}}
\newcommand{\lowerr}{R_1}
\newcommand{\upperr}{R_2}

\newcommand{\eventA}{\ev{E}}

\ifx \final \undefined
\def\final{1}  
\fi

\ifnum\final=0  
\newcommand{\lnote}[1]{[{\small Luis: \bf #1}]}
\newcommand{\hnote}[1]{[{\small Haolin: \bf #1}]}
\newcommand{\anonnote}[1]{[{\small anon: \bf #1}]}
\newcommand{\sidecomment}[1]{\marginpar{\tiny #1}}
\newcommand{\details}[1]{{\color{blue}\ [[#1]] }}
\else 
\newcommand{\lnote}[1]{}
\newcommand{\hnote}[1]{}
\newcommand{\anonnote}[1]{}
\newcommand{\sidecomment}[1]{}
\newcommand{\details}[1]{}
\fi  

\ifsiam
\title{Overcomplete order-3 tensor decomposition, blind deconvolution and Gaussian mixture models\thanks{Submitted to the editors DATE\funding{National Science Foundation Grants CCF-1657939, CCF-1422830, CCF-2006994 and CCF-1934568}}}
\author{Haolin Chen\thanks{Department of Mathematics, University of California, Davis (\email{hlnchen@ucdavis.edu}), (\email{lrademac@ucdavis.edu}).}\and Luis Rademacher\footnotemark[2]}
\else
\title{Overcomplete order-3 tensor decomposition, blind deconvolution and Gaussian mixture models}
\date{}
\author{Haolin Chen\\ University of California, Davis \\hlnchen@ucdavis.edu\and Luis Rademacher \\University of California, Davis\\lrademac@ucdavis.edu}
\fi

\begin{document}

\maketitle

\begin{abstract}
We propose a new algorithm for tensor decomposition, based on Jennrich's algorithm, and apply our new algorithmic ideas to blind deconvolution and Gaussian mixture models.
Our first contribution is a simple and efficient algorithm to decompose certain symmetric overcomplete order-3 tensors, that is, three dimensional arrays of the form $T = \sum_{i=1}^n a_i \otimes a_i \otimes a_i$ where the $a_i$s are not linearly independent.
Our algorithm comes with a detailed robustness analysis.
Our second contribution builds on top of our tensor decomposition algorithm to expand the family of Gaussian mixture models whose parameters can be estimated efficiently. 
These ideas are also presented in a more general framework of blind deconvolution that makes them applicable to mixture models of identical but very general distributions, including all centrally symmetric distributions with finite 6th moment.
\end{abstract}

\ifsiam

\fi

\section{Introduction}


Tensor decomposition is a basic tool in data analysis.
The \emph{order-3 tensor decomposition problem}\footnote{``Tensor decomposition'' here is a shorthand for a specific kind of tensor decomposition sometimes called tensor rank decomposition or canonical polyadic decomposition.} can be stated as follows:
Given an order-3 tensor $T = \sum_{i=1}^n a_i \otimes a_i \otimes a_i$, recover the vectors $a_i \in \RR^d$.
The problem is \emph{undercomplete} if the $a_i$s are linearly independent, otherwise it is \emph{overcomplete}.\details{note that this notions seem to depend on the $a_i$s and not just on $T$, but Kruskal guarantees that they only depend on $T$. Or does it?}
Two problems in data analysis motivate us here to study tensor decomposition:
blind deconvolution and Gaussian mixture models (GMM).

A \emph{deconvolution} problem can be formulated as follows: 
We have a $d$-dimensional random vector 
\begin{equation}\label{equ:deconvolution}
Y = Z + \eta 
\end{equation}
where $Z$ and $\eta$ are independent random vectors.
Given samples from $Y$, the goal is to determine the distribution of $Z$.
We call it \emph{blind deconvolution} when the distribution of $\eta$ is unknown, otherwise it is \emph{non-blind}.
It is called \emph{deconvolution} because the distribution of $Y$ is the convolution of the probability distributions of $Z$ and $\eta$.

The following \emph{mixture model parameter estimation problem} can be recast as a blind deconvolution problem:
Let $X$ be a $d$-dimensional random vector distributed as the following mixture model: First sample $i$ from $[n]$, each value with probability $w_i$ ($w_i > 0$, $\sum_i w_i = 1$), then let $X = \mu_i + \eta$, where $\eta$ is a given $d$-dimensional random vector and $\mu_i \in \RR^d$.
The estimation problem is to estimate $\mu_i$s and $w_i$s from samples of $X$.
It is a deconvolution problem $X = Z + \eta$ when $Z$ follows the discrete distribution equal to $\mu_i$ with probability $w_i$.
It is blind when the distribution of $\eta$ is unknown.

The \emph{GMM parameter estimation problem} can be described as follows:
Let $X \in \RR^n$ be a random vector with density function $x \mapsto \sum_{i=1}^k w_i f_i(x)$ where $w_i > 0$, $\sum_i w_i = 1$ and $f_i$ is the Gaussian density function with mean $\mu_i \in \RR^n$ and covariance matrix $\Sigma_i \in \RR^{n \times n}$.  
GMM parameter estimation is the following algorithmic question: 
Given iid.\ samples from $X$, estimate $w_i$s, $\mu_i$s and $\Sigma_i$s.

The GMM parameter estimation problem is a deconvolution problem when the covariance matrices of the components are the same, namely $\Sigma_i = \Sigma$.
Specifically, $X = Z + \eta$ where $Z$ follows a discrete distribution taking value $\mu_i$ with probability $w_i$, $i=1,\dotsc, k$ and $\eta$ is Gaussian with mean 0 and covariance $\Sigma$.
It is blind if $\Sigma$ is unknown.

While the undercomplete tensor decomposition problem is well-understood (based on algorithmic techniques such as the tensor power method and Jennrich's algorithm \cite{harshman1970foundations}), the overcomplete regime is much more challenging \cite[Chapter 7]{MAL-057}.
Within the overcomplete case, there are fewer techniques available for the order-3 case than there are for higher order \cite[Section 7.3]{MAL-057}.
We discuss some of these techniques and challenges below (\cref{sec:relatedwork}).

\subsection{Our results}





\paragraph{Overcomplete tensor decomposition} 
We propose an algorithm, based on Jennrich's algorithm, that can recover the components $a_i \in \RR^d$ to within error $\eps$ given a symmetric order-3 tensor $T = \sum_{i=1}^{d+k} a_i \otimes a_i \otimes a_i$, when any $d$-subset of the $a_i$s is linearly independent, in time polynomial in $d^k$, $1/\eps^k$ and natural conditioning parameters.
Note that our goal is to show that the running time has polynomial dependence in that sense and the error has inverse polynomial dependence but we do not optimize the degrees of the polynomials.
Even though the algorithm is exponential in $k$, the case $k=1$ already makes possible a new GMM result (see below).
Our algorithm (\cref{alg: approx tensor decomposition}) and its analysis (\cref{thm:tensor decomposition main theorem}) are stronger than the informal statement above in two important ways:
It is robust in the sense that it approximates the $a_i$s even when the input is a tensor that is $\eps'$-close to $T$.
Also, it turns out that parameter $k$ above, the number of $a_i$s beyond the dimension $d$, is not the best notion of overcompleteness.
In our result the tensor is of the form $T = \sum_{i=1}^{r+k} a_i \otimes a_i \otimes a_i$, where $r$ is the robust Kruskal rank of $a_i$s (informally the maximum $r$ such that any $r$-subset is well-conditioned, \cref{def:krank}), so that $k$ is the number of components above the robust Kruskal rank.
Thus, our analysis also applies when the Kruskal rank is less than $d$.


\paragraph{Blind deconvolution}	
We provide an efficient algorithm for the following blind deconvolution problem: 
Approximate the distribution of $Z$ (from \eqref{equ:deconvolution}) when it is a $d$-dimensional discrete distribution supported on $d$ points satisfying a natural non-degeneracy condition (\cref{assum: linear independence}), the distribution of $\eta$ is unknown and the first and third moments of $\eta$ are 0 with finite 6th moment (this includes the natural case where $\eta$ has a centrally symmetric distribution with finite 6th moment).
Equivalently, it can solve the mixture model parameter estimation problem above under the same conditions (\cref{alg: mixture learning algorithm,thm: mixture learning main theorem}).

\paragraph{GMM} 
We show an efficient algorithm for the following GMM parameter estimation problem:
Given samples from a $d$-dimensional mixture of $d$ identical and not necessarily spherical Gaussians with unknown parameters $w_i$, $\mu_i$, $\Sigma$, estimate all parameters (\cref{alg: gaussian mixture algorithm,thm: GMM main theorem}).

It may seem as if the last two contributions (blind deconvolution and GMM) could be attacked with standard \emph{undercomplete} tensor decomposition techniques given that the number of components is equal to the ambient dimension and therefore they could be linearly independent. 
It is not clear how that could actually happen, as the non-spherical unknown covariance seems to make standard approaches inapplicable and our contribution is a formulation that involves an overcomplete tensor decomposition and uses our overcomplete tensor decomposition algorithm in an essential way.



\subsection{Related work}\label{sec:relatedwork}

Among basic tensor decomposition techniques for the undercomplete case we have tensor power iteration (see \cite{MAL-057} for example) and Jennrich's algorithm (\cite{harshman1970foundations}, also know as simultaneous diagonalization and rediscovered several times, with variations credited to \cite{MR1238921}).
Tensor power iteration is more robust than Jennrich's algorithm, while
Jennrich's algorithm can be applied more generally: Tensor power iteration is mainly an algorithm for orthogonal tensors (orthogonal $a_i$s) and the general case with additional information, while Jennrich's algorithm can decompose the general case without additional information. 
Our contributions below are based on Jennrich's algorithm because of this additional power.
The robustness of Jennrich's algorithm is studied in several papers; our analysis builds on top of \cite{goyal2014fourier, 10.1145/2591796.2591881}. 

For the overcomplete regime we have algorithms such as FOOBI \cite{DBLP:journals/tsp/LathauwerCC07} and the work of \cite{10.1145/2591796.2591881,DBLP:journals/corr/AnandkumarGJ14, DBLP:conf/colt/AnandkumarGJ15, DBLP:conf/approx/GeM15, DBLP:conf/focs/MaSS16, DBLP:conf/nips/0001M17, DBLP:conf/colt/HopkinsSS19}.

Many techniques for the overcomplete case only make sense for orders 4 and higher or have weaker guarantees in the order-3 case.
For example, some techniques use the fact that a $d\times d\times d\times d$ tensor can be seen as an $d^2 \times d^2$ matrix (and similarly for order higher than 4), while no equally useful operation is available for order-3 tensors.
Nevertheless, there are several results about decomposition in the order-3 case that are relevant to our work: Kruskal's uniqueness of decomposition \cite{MR444690}, a robust version of Kruskal's uniqueness and an algorithm running in time exponential in the number of components \cite{bhaskara2014uniqueness}, an algorithm for tensors with incoherent components \cite{DBLP:journals/corr/AnandkumarGJ14, DBLP:conf/colt/AnandkumarGJ15}, a quasi-polynomial time algorithm (based on sum of squares) for tensors with random components \cite{DBLP:conf/approx/GeM15} and polynomial time algorithms (also based on sum of squares) for tensors with random components \cite{hopkins2016fast, ma2016polynomial}. 
Among works closest to ours, in \cite{MR3206990, MR3573806} an algorithm that is efficient in the mildly overcomplete case is proposed for overcomplete order-3 tensor decomposition under natural non-degeneracy conditions.
Though our results have similar assumptions and computational cost compared to \cite{MR3206990, MR3573806}, our algorithm is comparatively a very simple randomized algorithm and we provide a rigorous robustness analysis.


Blind deconvolution-type problems have a long history in signal processing and specifically in image processing as a deblurring technique (see, e.g., \cite{5963691}).
The idea of using higher order moments in blind identification problems is standard too in signal processing, specifically in Independent Component Analysis (see e.g. \cite{Comon2010, 150113}). 
Our model \eqref{equ:deconvolution} is somewhat different but very natural and inspired by mixture models. 

With respect to GMMs, we are interested in parameter estimation in high dimension with no separation assumption (i.e., the means $\mu_i$ can be arbitrarily close).
Among the most relevant results in this context we have the following polynomial time algorithms: 
\cite{hsu2013learning}, for linearly independent means and spherical components (each $\Sigma_i$ is a multiple of the identity); 
\cite{anderson2014more}, for $O(d^c)$ components with identical and known covariance $\Sigma$; 
\cite[Section 7]{DBLP:journals/corr/GoyalVX13}, \cite{goyal2014fourier}, for linearly independent means and spherical components in the presence of Gaussian noise;
\cite{ge2015learning}, for a general GMM with $O(\sqrt{d})$ components in the sense of smoothed analysis. 
Our algorithm expands the family of GMMs for which efficient parameter estimation is possible. 
It does not require prior knowledge of the covariance matrix unlike \cite{anderson2014more} and can handle more components ($d$ components) than \cite{ge2015learning} at the price of assuming all covariance matrices are identical. 
With respect to recent results on clustering-based algorithms \cite{DBLP:journals/corr/abs-2005-06417,jia2020robustly}, we consider these works incomparable to ours since clustering-based algorithms typically require some separation assumptions in the parameters.

\section{Notation and preliminaries}\label{sec:preliminaries}
    
    For clarity of exposition we analyze our algorithms in a computational model where we assume arithmetic operations between real numbers take constant time.
    We use the notation $\poly(\cdot)$ to denote a fixed polynomial that is non-decreasing in every argument. 
    See \cite[Section 5.3]{DBLP:journals/corr/GoyalVX13} for a discussion of the complexity of Jennrich's algorithm.
    
    For $n\in\NatureNumber$, let $[n] = \{1,\dotsc n\}$. 
    The unit sphere in $\Real^d$ is denoted by $\sphere^{d-1}$.
    
    \paragraph{Matrices and vectors} For a matrix $A\in\Real^{m\times n}$, we denote by $\sigma_i(A)$ its $i$-th largest singular value, by $A^\dagger$ its Moore-Penrose pseudoinverse, and by $\kappa(A) = \sigma_1(A)/\sigma_{\min(m,n)}(A)$ its condition number. 
    Let $\vecop(A)\in\Real^{mn}$ denote the vector obtained by stacking all columns of $A$.
    Denote by $\diag(a)$ the diagonal matrix with diagonal entries from $a$, where $a$ is a (column) vector. 
    Let $\norm{\cdot}_2$ denote the spectral norm of a matrix and $\norm{\cdot}_F$, the Frobenius norm of a matrix.
    
    In $\Real^d$, we denote by $\inner{a}{b}$ the inner product of two vectors $a,b$. 
    Let $\hat a = a/\enorm{a}$. 
    For a set of vectors $\{a_1,a_2,\dotsc,a_n\}$, we denote their linear span by $\spanop\{a_1,\dotsc,a_n\}$. 
    We use $[a_1,a_2,\dotsc,a_n]$ to denote the matrix containing $a_i$s as columns.
    If $A = [a_1,a_2,\dotsc, a_n]$, we have $\hat A = [\hat a_1,\hat a_2,\dotsc, \hat a_n]$ and $\tilde A$ follows a similar definition. 
    We denote by $A_m\in\Real^{d\times m}$ the matrix $[a_1,a_2,\dotsc,a_m]$ for some $m<n$ and by $A_{>m}\in\Real^{d\times(n-m)}$ the matrix $[a_{m+1},\dotsc,a_n]$. We say the matrix $A$ is $\rho$-bounded if $\max_{i\in[n]}\norm{a_i}_2\leq \rho$.
    Given a vector $a\in\Real^d$ or a diagonal matrix $D\in\Real^{d\times d}$, for $r\in\Real$, notations $a^r$ and $D^r$ are used for entry-wise power.

\begin{definition}[\cite{MR444690,bhaskara2014uniqueness}]\label{def:krank}
Let $A \in \RR^{m \times n}$. 
The \emph{Kruskal rank} of $A$, denoted $\krank(A)$, is the maximum $k \in [n]$ such that any $k$ columns of $A$ are linearly independent.
Let $\tau > 0$.
The \emph{robust Kruskal rank (with threshold $\tau$)} of $A$, denoted $\rkrank{\tau}{A}$, is the maximum $k \in [n]$ such that for any subset $S \subseteq [n]$ of size $k$ we have $\sigma_k(A_S) \geq 1/\tau$. 
\end{definition}
\paragraph{Tensors}
For a symmetric order-3 tensor $T \in \RR^{d\times d \times d}$ and a vector $x \in \RR^d$, let $T_x$ denote the matrix $\sum_{i,j,k=1}^d T_{ijk} x_i e_j e_k \in \RR^{d \times d}$. 
Let $a^{\otimes3}$ be a shorthand for $a\otimes a\otimes a$. For a rank $n$ symmetric order-3 tensor $T =  \sum_{i=1}^n a_i^{\otimes3}$, we say the tensor $T$ is $\rho$-bounded if $\max_{i\in[n]}\norm{a_i}_2\leq \rho$.


\paragraph{Cumulants}
The cumulants of a random vector $X$ are a sequence of tensors related to the moment tensors of $X$: $K_1(X), K_2(X), K_3(X), \dotsc$
We only state the properties we need, see \cite{mccullagh2018tensor} for an introduction.
We have: $K_1(X) = \expectation[X], K_2(X) = \cov(X), K_3(X) = \expectation[(X - \expectation[X])^{\otimes3}]$. 
\details{Note that the first three cumulants coincide with the expectation, the covariance matrix and the third central moment of $X$.}
Cumulants have the property that for two independent random variables $X,Y$ we have $K_m(X+Y) = K_m(X) + K_m(Y)$. 
The first two cumulants of a standard Gaussian random vector are the mean and the covariance matrix, all subsequent cumulants are zero.

\paragraph{Jennrich's algorithm \cite{harshman1970foundations}} 
The basic idea of Jennrich's algorithm to decompose a symmetric order-3 tensor with linearly independent components $a_1, \dotsc, a_d \in \RR^d$ is the following: for random unit vectors $x, y \in \RR^d$, compute the (right) eigenvectors of $T_x T_y^{-1}$. 
With probability 1, the set of eigenvectors is equal to the set of directions of $a_i$s (the eigenvectors recover the $a_i$s up to sign and norm).
We use a version that allows for the number of $a_i$s to be less than $d$ and that includes an error analysis \cite{DBLP:journals/corr/GoyalVX13,goyal2014fourier}.
\details{
\paragraph{Error propagation} We will frequently encounter the situation where we have ``If $X\leq \eps$, then $Y$ is less than some polynomial in terms of $\eps, d$ and other related parameters". To track the error for underlying $Y$ within each step of our algorithm, we will use subscripted symbols to indicate which step or theorem the underlying symbol is referring to. For example, $\eps_{\ref*{lem: deflation}}$ refers to the error term used in \cref{lem: deflation}.
}

\section{Overcomplete order-3 tensor decomposition}

   We consider the problem of decomposing (recovering $a_i$s) a symmetric order-3 tensor $T \in \Real^{d\times d \times d}$ of rank $n$:
    \begin{equation}
        \label{equ: order-3 tensor}
        T = \sum_{i\in[n]} a_i^{\otimes3}.
    \end{equation}
    When the $a_i$s are linearly independent, Jennrich's algorithm efficiently recovers them given $T$.
    But it has no guarantees if the components are linearly dependent. 
    Our main idea for the linearly dependent case is: it is still possible that a large subset $\{a_1, \dotsc, a_r\}$ of components is linearly independent, so if we cancel out the other components, $\{a_{r+1}, \dotsc, a_{n}\}$, the residual tensor can be efficiently decomposed via Jennrich's algorithm.
    To cancel the other components, we search for a vector $x$ orthogonal to them so that $T_x$ only involves the linearly independent components.
    A random or grid search for an approximately orthogonal $x$ is efficient if the number of components to cancel out is small.
    
    For clarity, we now describe an idealized version of our algorithm as if we had two vectors $x, y$ that are exactly orthogonal to the other components.
    (The actual algorithm uses a random search to find $x,y$.)
    We also want $x$, $y$ to be \emph{generic} with this orthogonality property, so that they can also play the roles of $x$, $y$ in Jennrich's algorithm (see \cref{sec:preliminaries}).
    Specifically, the genericity here is that the eigenvalues of $T_x T_y^{-1}$ are distinct.
    In that case, the eigendecomposition of $T_x T_y^{-1}$ recovers the directions of $\{a_1, \dotsc, a_r\}$.
    Then, a linear system of equations provides the lengths of $\{a_1, \dotsc, a_r\}$.
    Once $\{a_1, \dotsc, a_r\}$ is recovered, the components of $T$ associated to them can be removed from $T$ (deflation) and Jennrich's algorithm can be applied a second time to the residual tensor to recover $\{a_{r+1}, \dotsc, a_{n}\}$.

    \subsection{Approximation algorithm and main theorem}
        \label{sec:main theorem}


        In the previous discussion we argued that given $x$, $y$ and $T$ with exact properties one can decompose $T$.
        In this subsection, we show that by repeatedly trying random choices we can find $x,y$ nearly orthogonal to $a_{r+1},\ldots,a_{r+k}$. 
        In practice, instead of the true tensor $T$, we usually have only an approximation $\tilde T$ of it and to be effective in this situation our algorithm comes with a robustness analysis that shows that if $\tilde T$ is close to $T$ then the output is close to the true components of $T$.         
        Our formal statements are \cref{alg: approx tensor decomposition} and \cref{thm:tensor decomposition main theorem}.

        \begin{algorithm}
            \caption{JENNRICH \cite[Algorithm ``Diagonalize"]{DBLP:journals/corr/GoyalVX13}}
            \label{alg: Jennrich}
            \begin{algorithmic}[1]
                \Statex{\textbf{Inputs:} $M_\mu, M_\lambda \in \Real^{d\times d}$, number of vectors $r$.}
                
                \State{compute the SVD of $M_\mu = VDU^\top$. Let $W$ be matrix whose columns are the left singular vectors (columns of $V$) corresponding to the top $r$ singular values;}
                
                \State{compute $M = (W^\top M_\mu W)(W^\top M_\lambda W)^{-1}$;}
                
                \State{compute the eigendecomposition: $M = P\Lambda P^{-1}$;}
                
                \Statex{\textbf{Outputs:} columns of $WP$.}
            \end{algorithmic}
        \end{algorithm}
        
        \begin{algorithm}
            \caption{Approximate tensor decomposition}
            \label{alg: approx tensor decomposition}
            \begin{algorithmic}[1]
                \Statex{\textbf{Inputs:} tensor $\tilde T \in \RR^{d\times d \times d}$, 
                error tolerance $\eps$, 
                tensor rank $n$, overcompleteness $k$, upper bound $M$ on $\norm{a_i}_2$ for $i\in[n]$. Let $r=n-k$ (Kruskal rank).}
                
                \Repeat
                    \State{pick $x,y$ iid.\ uniformly at random in $\sphere^{d-1}$;}
                    
                    \State{invoke \cref{alg: Jennrich} with $\tilde T_x, \tilde T_y$ and $r$. Denote the outputs by $\tilde a_i$ for $i\in[r]$;\label{line:decomposition}}
                    
                    \State{solve the least squares problem: $\min_{\xi_{1},\dotsc, \xi_{r}} \norm{\tilde A_r \diag(\xi_i\inner{x}{\tilde a_i})  \tilde A_r^\top - \tilde T_x}_2$.\label{line:leastsquares}}
                    
                    
                    \State{set $R = \tilde{T} - \sum_{i\in[r]}\xi_i\Tilde{a}_i^{\otimes3}$;}
                    
                    \State{pick $x',y'$ iid.\ uniformly at random in $\sphere^{d-1}$;}
                    
                    \State{invoke \cref{alg: Jennrich} with $R_{x'}, R_{y'}$ and $k$. Denote the outputs by $\tilde a_{r+i}$ for $i\in[k]$;\label{line:decomposition2}}
                    
                    \State{solve the least squares problem: $\min_{\xi_{r+1},\dotsc, \xi_{r+k}} \norm{\tilde A_{>r} \diag(\xi_{r+i} \inner{x'}{\tilde a_{r+i}}) \tilde A_{>r}^\top - R_{x'}}_2$. \label{line:leastsquares2}}
                    
                    
                    \State{reconstruct the tensor $T' = \sum_{i\in[r+k]} \xi_i \Tilde{a}_i^{\otimes3}$;}
                    
                \Until{$\sfnorm{T' - \tilde T} \leq \eps$, $\max_{i\in[r+k]} \abs{\xi_i}^{1/3}\leq 2M$}
                
                \Statex{\textbf{Outputs:} $\check{a}_i := \xi_i^{1/3}\Tilde{a}_i$, for $i \in [r+k]$.}
            \end{algorithmic}
        \end{algorithm}

        \begin{theorem}[Correctness of \cref{alg: approx tensor decomposition}]
        \label{thm:tensor decomposition main theorem}
        \label{thm:main tensor}
        Let $T = \sum_{i\in[r+k]} a_i^{\otimes 3}$, $1 \leq k \leq (r-2)/2$, and $a_i \in \RR^d$. 
        Let $A=[a_1,\dotsc,a_{r+k}]$ and $\rkrank{\tau}{A} \geq r$.
        Let $\tau > 0$, $M \geq \max_{i\in[r+k]}\norm{a_i}_2$, $ 0 < m \leq \min_{i\in[r+k]} \norm{a_i}_2$ and $0 < \eps_{out} \leq\min\{1, m^3\}$.
        There exist  polynomials $\poly_{\ref*{thm:main tensor}}(d,\tau,M), \poly'_{\ref*{thm:main tensor}}(d,\tau,M, m^{-1})$, such that if $\eps_{in} \leq \eps_{out}/\poly'_{\ref*{thm:main tensor}}$
        and $\tilde T$ is a tensor such that $\sfnorm{T - \tilde T} \leq \eps_{in}$, then \cref{alg: approx tensor decomposition} on input $\tilde T$ and $\eps = \eps_{out}/\poly_{\ref*{thm:main tensor}}$, outputs vectors $\check a_1,\dotsc,\check a_{r+k}$ such that for some permutation $\pi$ of $[r+k]$, we have  
        \(
        {\snorm{a_{\pi(i)} - \check a_i}}_2 \leq \eps_{out}\), \(\forall i\in[r+k]\).
        The expected running time is at most $\poly(d^k, \eps_{out}^{-k}, \tau^k, M^k, m^{-k})$.
        \end{theorem}

        
        \begin{proofidea}
        The proof has three parts. 
        First we show that if $T'$ (with which the algorithm finishes) is close to $\tilde T$ and has bounded components, then the components of $T'$, $\{\check{a}_i = \xi_i^{1/3}\tilde a_i\suchthat i\in[r+k]\}$, are close to those of $T$. 
        In the second part we show that, assuming good $x,y,x',y'$ have been found, the algorithm indeed finishes with a tensor $T'$ that is close to $\tilde T$ (and therefore, close to $T$ via triangle inequality), and how the error propagates. 
        In the third part we show the probabilistic bounds that guarantee efficient search of good $x, y, x', y'$.
        
        The first part follows from \cite[Theorem 2.6]{bhaskara2014uniqueness} (the version we need is \cref{thm: bhaskara} here).
        
        For the second part, we will assume that we have found good vectors $x,y$ that are nearly orthogonal to $k$ components, $\hat a_{r+1},\ldots,\hat a_{r+k}$. 
        \Cref{thm:gvx} (from \cite{goyal2014fourier}) and \cref{lem: first decomposition} guarantee that we can simultaneously diagonalize matrices $\tilde T_x$ and $\tilde T_y$ using Jennrich's algorithm (\cref{alg: Jennrich}), and the outputs are close to the directions of $a_i$s. 
        \cref{lem: norm theorem} shows that we can recover approximately the lengths of $a_i$s by solving a least squares problem once we have the directions. 
        At this point we completed the recovery of $r$ components. 
        \cref{lem: deflation} shows that when the deflation error is small, the residual tensor $R$ can be decomposed in the same way and the last $k$ directions are recovered. 
        At the end of the second part, \cref{lem: norm theorem 2} shows that the lengths of the last $k$ components are approximately recovered.

        The third part is shown in \cref{lem: tensor probability bound 4,lem: tensor probability bound 5}.
        \end{proofidea}
        
       The proof of \cref{thm:tensor decomposition main theorem} is deferred to \Cref{sec: proof}.

\section{Blind deconvolution of discrete distribution}\label{sec:blinddeconvolution}
    
    In this section we provide an application of \cref{alg: approx tensor decomposition}: to perform blind deconvolution of an additive mixture model of the form
    \begin{equation}
        \label{equ: mixture}
        X = Z + \eta
    \end{equation}
    in $\Real^d$, where $Z$ follows a discrete distribution that takes value $\mu_i$ with probability $w_i$ for $i\in[d]$, and $\eta$ is an unknown random variable independent of $Z$ with zero mean, zero 3rd moment and finite 6th moment. 

    
    
    Our goal is to recover the parameters of $Z$ when given samples from $X$.
    By estimating the overall mean and translating the samples we can, without loss of generality, assume that $\sum_{i\in[d]}w_i\mu_i = 0$ for the rest of this section.
    
    First we see that the parameters of $Z$ are identifiable from the 3rd cumulant of $X$ as the first and third moments of $\eta$ are zero. 
    Let $K_m(X)$ be the $m$-th cumulant of $X$. 
    By properties of cumulants (see \cref{sec:preliminaries}):
    \begin{equation}
        \label{equ: cumulant tensor}
        K_3(X) = K_3(Z)  + K_3(\eta) = \sum_{i\in[d]} w_i \mu_i^{\otimes3}. 
    \end{equation}
    
    
    If one decomposes $K_3(X)$, then the function $w_i^{1/3}\mu_i$ of the centers $\mu_i$ and the mixing weights $w_i$ is recovered.
    However the component vectors satisfy $\sum_i w_i \mu_i = 0$ (they are \emph{always} linearly dependent) and therefore applying Jennrich's algorithm naively has no guarantee.\footnote{Note that even when the overall mean is non-zero and the means are linearly independent, $K_3$ still has linearly dependent components as it is the \emph{central} 3rd moment. If one does not use $K_3$, then one loses \eqref{equ: cumulant tensor}.} 
    We show that, under the following non-degeneracy condition, our overcomplete tensor decomposition algorithm (\cref{alg: approx tensor decomposition}) works successfully.
    \begin{assumption}
        \label{assum: linear independence}
        $\rkrank{\tau}{[\mu_1, \dotsc, \mu_d]} = d-1$.
    \end{assumption}
    Under \cref{assum: linear independence}, we can decompose \cref{equ: cumulant tensor} with \cref{alg: approx tensor decomposition}.
    For simplicity, we reformulate the problem: let $a_i = \hat \mu_i$, and $\rho_i = \norm{\mu_i}_2$, our goal becomes
    to decompose $T = \sum_{i\in[d]} w_i\rho_i^3a_i^{\otimes3}$ subject to $\sum_{i\in[d]}w_i = 1$ and $\sum_{i\in[d]} w_i\rho_ia_i = 0$.

    We now state our algorithm (\cref{alg: mixture learning algorithm}) for blind deconvolution of discrete distribution.
    \begin{algorithm}
        \caption{Blind deconvolution of discrete distribution}
        \label{alg: mixture learning algorithm}
        \begin{algorithmic}[1]
            \Statex{\textbf{Inputs:} iid.\ samples $x_1, \dotsc, x_N$ from mixture $X$, 
            error tolerance $\eps'$, upper bound $\rho_{max}$ on $\norm{\mu_i}_2$ for $i\in[d]$, lower bound $w_{min}$ on $w_i$ for $i\in[d]$, robust Kruskal rank threshold $\tau$.}  
            
            \State{compute the sample 3rd cumulant $\Tilde{T}$ using \cref{fact: sample cumulant};\label{line: cumulant}}
            
            \State{invoke \cref{alg: approx tensor decomposition} with error tolerance $\eps_{\ref*{thm: mixture learning main theorem}} = \eps'/\poly_{\ref*{thm: mixture learning main theorem}}$, tensor rank $d$ and overcompleteness $1$ to decompose $\Tilde{T}$, thus obtain $\Tilde{a}_i \xi_i^{1/3} = \check a_i$ for $i\in[d]$;}
            
            \State{set $\Tilde{v}$ to the right singular vector associated with the minimum singular value of $\check A$;}
            
            \State{set $\Tilde{w}= \Tilde{v}^{3/2}/(\sum_{i\in[d]}\Tilde{v}_i^{3/2})$, 
            $\Tilde{\mu}_i = \tilde w_i^{-1/3} \check{a}_i$ for $i\in[d]$;}
            
            
            \Statex{\textbf{Outputs:} estimated mixing weights $\Tilde{w}_1,\dotsc,\tilde w_d$, and estimated means $\Tilde{\mu}_1,\dotsc,\tilde \mu_d$.}
        \end{algorithmic}
    \end{algorithm}
    
    \begin{theorem}[Correctness of \cref{alg: mixture learning algorithm}]
        \label{thm: mixture learning main theorem}
        Let $X = (X_1, \dotsc, X_d) = Z + \eta$ be a random vector as in \cref{equ: mixture} satisfying \cref{assum: linear independence}.
        Assume $0< w_{min} \leq \min_{i\in[d]}w_i$, $\rho_{max} \geq \max_{i\in[d]}\rho_i$, $0< \rho_{min} \leq \min_{i\in[d]}\rho_i$, $0< \eps' \leq \min\{1, w_{min}\rho_{min}^3\}$ and $\delta \in (0,1)$. There exists a polynomial $\poly_{\ref*{thm: mixture learning main theorem}}(d,\tau,\rho_{max},w_{min}^{-1})$ such that if $\eps_{\ref*{thm: mixture learning main theorem}} = \eps'/\poly_{\ref*{thm: mixture learning main theorem}}$,
        then given $N$ iid.\ samples of $X$, with probability $1-\delta$ over the randomness in the samples, \cref{alg: mixture learning algorithm} outputs $\Tilde{\mu}_1,\ldots,\Tilde{\mu}_{d}$ and $\Tilde{w}_1,\ldots,\Tilde{w}_{d}$ such that for some permutation $\pi$ of $[d]$ and for all $i \in [d]$ we have \(
            \norm{\mu_{\pi(i)} - \Tilde{\mu}_i}_2\leq \eps'\) and  
            \( \abs{w_{\pi(i)} - \Tilde{w}_i} \leq \eps'
        \).
        The expected running time over the randomness of \cref{alg: approx tensor decomposition} is at most $\poly(d,\eps'^{-1},\delta^{-1},\tau,\rho_{max},\rho_{min}^{-1},w_{min}^{-1}, \max_i\expectation[X_i^6])$ and will use 
        $N = \Omega \bigl(\eps'^{-2}\delta^{-1}d^{11}\max_{i\in [d]}\expectation [X_i^6] \bigl(\poly'_{\ref*{thm:main tensor}}(d,\tau,\rho_{max}, w_{min}^{-1/3}\rho_{min}^{-1}) \bigr)^2 \bigr)$ samples.
    \end{theorem}
    The proof of \cref{thm: mixture learning main theorem} has two parts. 
    First, we show that the 3rd cumulant can be estimated to within $\eps$ accuracy with polynomially many samples.
    This follows from a standard argument using $k$-statistics.
    The second part is about the tensor decomposition.
    Note that \cref{thm:tensor decomposition main theorem} guarantees that we can recover $\Tilde{a}_i$ approximately in the direction of $a_{\pi(i)}$ and $\xi_i$ close to $w_{\pi(i)}\rho_{\pi(i)}^3$ for some permutation $\pi$. 
    However we are not finished yet as our goal is to recover both the centers and the mixing weights. Therefore we need to decouple $w_i$ and $\rho_i$ from $w_i\rho_i^3$, which corresponds to steps 3 and 4 in \cref{alg: mixture learning algorithm}.

    \paragraph{3rd cumulant estimation} 
    The details are in \cref{sec: sample cumulants}, we only give the main result here:
    \begin{lemma}[Estimation of the 3rd cumulant]
        \label{lem: finite sample analysis}
        Let $T,\tilde T$ be the 3rd cumulant of $X= (X_1, \dotsc, X_d)$ and its unbiased estimate ($k$-statistic) using \cref{fact: sample cumulant}, respectively.
        Given any $\varepsilon,\delta\in(0,1)$, and $N =  \Omega(d^9\varepsilon^{-2}\delta^{-1}\max_{i\in [d]}\expectation[X_i^6]\})$, with probability $1-\delta$ we have $\norm{T-\tilde T}_F \leq \varepsilon$.
    \end{lemma}
    \begin{proof}
        Apply \cref{lem: cumulant sample complexity} with accuracy $\varepsilon/d^3$, failure probability $\delta/d^3$ and taking the union bound over $d^3$ entries, to see that $N = \Omega(d^9\varepsilon^{-2}\delta^{-1}\max_{i\in [d]}\expectation[X_i^6]\})$ samples are sufficient.
    \end{proof}
    
    \paragraph{Decoupling}
   We will decouple the mixing weights $w_i$ and the norms $\rho_i$ after we decompose the tensor $\tilde T$. 
   As $\expectation[X] = 0$, the true parameters satisfy
    $
        \sum_{i\in[d]}w_i\rho_ia_i = 0
    $,
    which can be reformulated as a linear system 
    \begin{equation}
        \label{weight recovery linear system}
        AD_{w_i\rho_i^3}^{1/3}w^{2/3} = 0,
    \end{equation}
    where $D_{w_i\rho_i^3} = \diag(w_i\rho_i^3)$ and $A$ contains $a_i$s as columns. 
    To decouple these parameters in the noiseless setting, one only needs to solve this system under the constraint that $w$ is a probability vector. As $\rank(A) = d-1$, $w$ will be uniquely determined.
    In other words, $w^{2/3}$ lies in the direction of the right singular vector associated with the only zero singular value. 
    It is natural then to recover the weights using our approximations to terms in the linear system, namely in the direction of the right singular vector associated to the minimum singular value of $\Tilde{A}D_{\xi}^{1/3}$, where $\tilde A = [\tilde a_1,\dotsc,\tilde a_d]$ and $D_\xi = \diag(\xi_i)$. 
    The following theorem guarantees this will work:
    \begin{theorem}[Decoupling]
        \label{thm: decoupling}
        Let $0 < w_{min} \leq \min_{i\in[d]}w_i$, and $\rho_{max} \geq \max_{i\in[d]}\norm{\mu_i}_2$.
        Suppose the outputs of step 2 in \cref{alg: mixture learning algorithm}, namely $\xi_1,\dotsc,\xi_d$ and $\Tilde{A} = [\tilde a_1,\dotsc,\tilde a_d]$, satisfy \cref{thm:tensor decomposition main theorem} with 
        \(
            \eps_{out} < \hfrac{w_{min}^{4/3}}{(24d\tau)}
        \)
        and permutation $\pi$. 
        One can choose 
        positive right singular vectors $v,\Tilde{v}$ associated with the minimum singular value of $AD_{w_i\rho_i^3}^{1/3}, \Tilde{A}D_{\xi}^{1/3}$, respectively. 
        Define 
        \(
            \Tilde{w} = \hfrac{\Tilde{v}^{3/2}}{\sum_{i\in[d]}\Tilde{v}_i^{3/2}}
        \)
        and
        \(    
            \Tilde{\rho}_i = (\xi_i/\Tilde{w}_i)^{1/3}
        \). 
        Then
        \(
            \abs{w_{\pi(i)} - \tilde w_i} \leq 12w_{min}^{-1/3}d\tau\eps_{out}
        \) and
        \(
            \abs{\rho_{\pi(i)} - \Tilde{\rho}_i}\leq 48w_{min}^{-4/3}\rho_{max}d\tau\eps_{out}
        \).
    \end{theorem}
    \begin{proof}
        \details{$M=(\mu_1, \dotsc, \mu_d)$, $AD_{w_i\rho_i^3}^{1/3} = M\diag(w)^{1/3}$, $\sigma_{d-1}(M_{-d})\geq 1/\tau$, $A=\hat M$.}
        We start by showing that $v, \tilde v$ and $\tilde w$ are well-defined.
        Since $w^{2/3}$ is a solution to \cref{weight recovery linear system} and $AD^{1/3}_{w_i\rho_i^3}$ is of rank $d-1$, we pick $v = w^{2/3}/\norm{w^{2/3}}_2$. 
        To show that $\tilde v$ is well-defined, first we bound the singular values and vectors of $\Tilde{A}D_{\xi}^{1/3}$. 
        Let $\tilde \sigma_i = \sigma_i(\Tilde{A}D_{\xi}^{1/3})$.
        By \cref{thm: weylsvd},
        \begin{equation}
            \label{equ: decoupling bound 1}
            \begin{split}
                \tilde \sigma_d &\leq \norm{AD_{w_i\rho_i^3}^{1/3} - \Tilde{A}D_{\xi}^{1/3}}_2 
                \leq \sqrt{d} \eps_{out} 
                < \hfrac{w_{min}^{4/3}}{(24\sqrt{d}\tau)}.
            \end{split}
        \end{equation}
        
        To obtain the deviation in the singular vectors, we first show that $\tilde \sigma_1, \dotsc, \tilde \sigma_{d-1}$ are bounded away from zero. 
        Let $\Sigma_1 = \diag \bigl( \sigma_1(AD^{1/3}_{w_i\rho_i^3}),\dotsc,\sigma_{d-1}(AD^{1/3}_{w_i\rho_i^3}) \bigr)$, $\tilde \Sigma_1 = \diag(\tilde \sigma_1,\ldots,\tilde \sigma_{d-1})$ and $\Delta = \hfrac{w_{min}^{1/3}}{(2\tau)}$. Suppose $\hat \sigma_{d-1}$ is the least singular value of the matrix obtained by deleting the first column of $AD^{1/3}_{w_i\rho_i^3}$, then it follows that
        \(
            \sigma_{d-1}(AD^{1/3}_{w_i\rho_i^3}) \geq \hat \sigma_{d-1} \geq \hfrac{w_{min}^{1/3}}{\tau}
        \),
        where the first inequality follows from the interlacing property of singular values of a matrix and its submatrix obtained by deleting any column, and the second inequality comes from \cref{assum: linear independence}. The minimum diagonal term in $\tilde \Sigma_1$ satisfies:
        \[
            \min_i\ (\tilde \Sigma_{1})_{ii} \geq \sigma_{d-1}(AD^{1/3}_{w_i\rho_i^3}) - \sqrt{d}\eps_{out} \geq \frac{w_{min}^{1/3}}{\tau} - \frac{w_{min}^{4/3}}{24\sqrt{d}\tau}\geq \frac{w^{1/3}_{min}}{2\tau} = \Delta.
        \]
        Therefore by \cref{thm: wedin} with $\Sigma_2 = 0$, we have for the singular vectors\footnote{Note that even though \cref{thm: wedin} gives the angle between the subspaces spanned by the first $d-1$ right singular vectors of $AD^{1/3}_{w_i\rho_i^3}$ and their perturbed counterparts, the same bound applies to the orthogonal complement, spanned by $v$.}:
        \begin{equation*}
             \norm{v-\Tilde{v}}_2 \leq \hfrac{\sqrt{2d}\eps_{out}}{\Delta} = 2\sqrt{2d}w_{min}^{-1/3}\tau\eps_{out}.
        \end{equation*}
        We get
         \( 
         \tilde v_{i} \geq v_i -  2\sqrt{2d}w_{min}^{-1/3}\tau\eps_{out} \geq w^{2/3}_{min} -  2\sqrt{2d}w_{min}^{-1/3}\tau\eps_{out} >0
         \),
         where the second inequality follows from $\sum_{i\in[d]} v_i^{3/2} \geq \sum_{i\in[d]} v_i^{2} = 1$.
         Hence $\tilde v$ also has positive entries and $\tilde w$ is well-defined.
         
         We now derive the bounds on the mixing weights and norms. Without loss of generality $\pi$ is the identity. 
         The mixing weight error is bounded by:
         \begin{equation}
            \label{equ: decoupling bound 2}
            \begin{split}
                 \norm{\Tilde{w}-w}_2 
                 &= \Biggl\lVert \frac{\Tilde{v}^{3/2}}{\sum_{i\in[d]}\Tilde{v}_i^{3/2}} - \frac{v^{3/2}}{\sum_{i\in[d]}v_i^{3/2}} \Biggr\rVert_2 \\
                 &\leq \frac{\norm{\Tilde{v}^{3/2}-v^{3/2}}_2}{\sum_{i\in[d]}v_i^{3/2}} + \frac{\norm{\Tilde{v}^{3/2}}_2}{(\sum_{i\in[d]}v_i^{3/2})(\sum_{i\in[d]}\Tilde{v}_i^{3/2})} \Biggabs{\sum_{i\in[d]}(v_i^{3/2}-\Tilde{v}_i^{3/2})}.
            \end{split}
         \end{equation}
         We bound each term in \cref{equ: decoupling bound 2} below, since $\tilde v, v$ both have entries in $(0,1]$:
         $ \sum_{i\in[d]} v_i^{3/2} \geq \norm{v}^2_2 = 1$,
         $\sum_{i\in[d]} \tilde v_i^{3/2} \geq \norm{\tilde v}^2_2 = 1$, and $\norm{\tilde v^{3/2}}_2 = (\sum_{i\in[d]}\tilde v_i^3)^{1/2} \leq \norm{\tilde v}_2 = 1$. Moreover:
         \begin{align*}
            &\norm{\Tilde{v}^{3/2}-v^{3/2}}_2 
                = \Bigl(\sum_{i\in[d]} \bigl(\tilde v_i^{3/2} - v_i^{3/2}\bigr)^2\Bigr)^{1/2} 
                \leq \frac{3}{2}\Bigl(\sum_{i\in[d]} (\tilde v_i-v_i)^2 \Bigr)^{1/2} 
                = \frac{3}{2}\norm{\tilde v - v}_2, \\
            &\Bigl\lvert \sum_{i\in[d]}(v_i^{3/2}-\Tilde{v}_i^{3/2}) \Bigr\rvert
                \leq \sum_{i\in[d]}\abs{v_i^{3/2}-\Tilde{v}_i^{3/2}} 
                \leq \frac{3}{2} \sum_{i\in[d]}\abs{v_i-\tilde v_i} 
                \leq  \frac{3\sqrt{d}}{2}\norm{\tilde v-v}_2,
         \end{align*}
         where the above two inequalities follow from $\abs{x^{3/2}-y^{3/2}}\leq 3\abs{x-y}/2$ for $x,y\in[0,1]$. 
         
        We obtain the following bound on the error in mixing weights:
        \begin{equation}
        \label{equ: decoupling bound 4}
        \norm{\Tilde{w} - w}_2 
        \leq (3/2)(1+\sqrt{d}) \norm{\tilde v-v}_2\leq 3\sqrt{2}w_{min}^{-1/3}(d+\sqrt{d})\tau\eps_{out} \leq 12w_{min}^{-1/3}d\tau\eps_{out}.
        \end{equation}
        Notice that our assumption on $\eps_{out}$ guarantees that $\tilde w_i \geq w_{min}/2$, therefore the error in the norm is bounded by:
        \begin{equation*}
             \begin{split}
                 \abs{\Tilde{\rho}_i-\rho_i}
                 &= \abs{(\xi_i/\Tilde{w}_i)^{1/3}-\rho_i} 
                 \leq \tilde w_i^{-1/3}\bigl(\abs{\xi_i^{1/3}-w_i^{1/3}\rho_i} +\rho_i\abs{w_i^{1/3}-\tilde w_i^{1/3}}\bigr) \\
                 &\leq \tilde w_i^{-1/3} \bigl(\eps_{out} + \rho_{max}\abs{w_i^{1/3}-\tilde w_i^{1/3}} \bigr) 
                 \leq (2w^{-1}_{min})^{1/3}(\eps_{out} + \rho_{max}w_{min}^{-2/3}\abs{w_i - \tilde w_i}) 
                 \\
                 &\leq 48w_{min}^{-4/3}\rho_{max}d\tau\eps_{out}
             \end{split}
         \end{equation*}
          where the second inequality comes from \cref{thm:tensor decomposition main theorem}, the third inequality comes from the fact $\abs{x^{1/3} - y^{1/3}}/\abs{x-y} \leq y^{-2/3}$ for any $x,y>0$, and the last follows from \cref{equ: decoupling bound 4}.
    \end{proof}
    
    We are now ready to prove \cref{thm: mixture learning main theorem}.
    \begin{proof}[Proof of \cref{thm: mixture learning main theorem}]
       Set the arguments of $\poly'_{\ref*{thm:main tensor}}, \poly_{\ref*{thm:main tensor}},$ to $(d,\tau, \rho_{max}, w_{min}^{-1/3}\rho_{min}^{-1})$ and $(d,\tau,\rho_{max})$, respectively. Assume for a moment that $N$ is large enough so that $\tilde T$ in step \ref{line: cumulant} satisfies $\norm{T - \tilde T}_F \leq \eps_{out}/ (\poly'_{\ref*{thm:main tensor}})$ and we can apply \cref{thm:main tensor}.We start by verifying that we can apply \cref{thm: decoupling}.        
        Set
        \(
             \poly_{\ref*{thm: mixture learning main theorem}} = 49w_{min}^{-4/3}\max\{\rho_{max},1\}d\tau\poly_{\ref*{thm:main tensor}}
        \). 
        By \cref{thm:main tensor}, our choice of $\eps_{\ref*{thm: mixture learning main theorem}}$ guarantees that the output error of step 2 in \cref{alg: mixture learning algorithm} is $\eps_{out} = \eps_{\ref*{thm: mixture learning main theorem}}\poly_{\ref*{thm:main tensor}} = \eps' / (49w_{min}^{-4/3}\max\{\rho_{max},1\}d\tau) < \hfrac{w_{min}^{4/3}}{(24d\tau)}$ (using our assumption $\eps'\leq 1$).
        We now bound our estimation error for $\norm{\mu_i}_2$ and $w_i$ with \cref{thm: decoupling}. Assuming the permutation is the identity we have for $i\in[d]$:
        \begin{equation*}
            \begin{split}
                \norm{\mu_i - \tilde\mu_i}_2 &\leq \abs{\rho_i - \tilde \rho_i}\norm{\tilde a_i}_2 + \rho_i\norm{a_i - \tilde a_i}_2 
                \leq (48w_{min}^{-4/3}\rho_{max}d\tau+\rho_{max})\eps_{out}\leq \eps',\\
                \abs{w_i - \tilde w_i} &\leq 12w_{min}^{-1/3}d\tau\eps_{out} \leq \eps'.
            \end{split}
        \end{equation*}
        Next, we derive the sample complexity: we need $\norm{T - \tilde T}_F \leq \eps_{in} \leq \eps_{out}/(\poly'_{\ref*{thm:main tensor}}) = \eps'\poly_{\ref*{thm:main tensor}}/(\poly'_{\ref*{thm:main tensor}}\poly_{\ref*{thm: mixture learning main theorem}})$.
        By \cref{lem: finite sample analysis}, $N=\Omega(\eps'^{-2}\delta^{-1}d^{11}\max_{i\in [d]}\expectation[X_i^6](\poly'_{\ref*{thm:main tensor}})^2)$ many samples are sufficient for $\eps_{in}$ to meet the assumption. 
        Since $N$ is polynomial in $\delta^{-1}$ and $\max_{i\in[d]} \expectation[X_i^6]$, the expected running time will also be polynomial in them.
    \end{proof}
\section{Parameter estimation of Gaussian mixture models (GMM)}
    In this section we consider a specific family of mixture models, namely GMM with identical but unknown covariance matrices.
    The model is as in \cref{equ: mixture}, where $\eta \sim \mathcal{N}(0,\Sigma)$.
    Our goal is to approximate all parameters of the mixture: $\Sigma$, $w_i$s and $\mu_i$s.
    Again, suppose \cref{assum: linear independence} holds and the mean of the mixture is zero (by translating the samples as in \cref{sec:blinddeconvolution}). 
    \Cref{alg: mixture learning algorithm} guarantees that we can recover the mixing weights $w_i$s and centers $\mu_i$s of $Z$. 
    To recover $\Sigma$, notice that since the mean is zero,
    \( 
        \cov(X) = \expectation[XX^\top] = \sum_{i\in[d]} w_i \mu_i\mu_i^\top + \Sigma
    \).
    The covariance matrix can be approximated then by taking the difference between the sample second moment of $X$ and the second moment of the reconstructed discrete distribution. 
    We make this precise in \cref{alg: gaussian mixture algorithm} and \cref{thm: GMM main theorem}.
    \begin{algorithm}
        \caption{Parameter estimation for GMM}
        \label{alg: gaussian mixture algorithm}
        \begin{algorithmic}[1]
            \Statex{\textbf{Inputs:} iid.\ samples $x_1,\ldots,x_N$ from mixture $X$, 
            error tolerance $\eps''$, upper bound $\rho_{max}$ on $\norm{\mu_i}_2$ for $i\in[d]$, lower bound $w_{min}$ on $w_i$ for $i\in[d]$, robust Kruskal rank threshold $\tau$.}
            
            \State{invoke \cref{alg: mixture learning algorithm} with samples from $X$ and parameters $\eps' = \eps''/\poly_{\ref*{thm: GMM main theorem}}$, $\rho_{max}, w_{min}, \tau$ to get $\tilde w_i$ and $\tilde \mu_i$ for $i\in[d]$;}
            
            \State{set $\tilde \Sigma = \frac{1}{N}\sum_{j\in[N]}x_jx_j^\top - \sum_{i\in[d]}\tilde w_i \tilde\mu_i\tilde\mu_i^\top$;}
            
            \Statex{\textbf{Outputs:} estimated covariance matrix $\tilde \Sigma$, mixing weights and means $\tilde w_i, \tilde \mu_i\suchthat i\in[d]$.}
        \end{algorithmic}
    \end{algorithm}
    \begin{theorem}[Correctness of \cref{alg: gaussian mixture algorithm}]
        \label{thm: GMM main theorem}
        Let $X$ be a GMM with identical but unknown covariance matrices satisfying \cref{assum: linear independence}.
        Assume $0 < w_{min} \leq \min_{i\in[d]}w_i$, $\rho_{max} \geq \max_{i\in[d]}\rho_i$, $0<\rho_{min}\leq\min_{i\in[d]}\rho_i$, 
        $0< \eps'' \leq \min\{1, w_{min}\rho_{min}^3\}$ and $\delta\in(0,1)$. 
        There exist a polynomial $\poly_{\ref*{thm: GMM main theorem}}(d,\rho_{max})$ such that if $\eps'= \eps''/\poly_{\ref*{thm: GMM main theorem}}$,
        then given $N$ iid.\ samples of $X$ and with probability $1-\delta$ over the randomness in the samples \cref{alg: gaussian mixture algorithm} outputs $\tilde \mu_1,\dotsc,\tilde \mu_d$, $\tilde w_1,\dotsc,\tilde w_d$ and $\tilde \Sigma$ such that for some permutation $\pi$ of $[d]$ and $\forall i\in[d]$:
        \(
        \norm{\tilde \Sigma - \Sigma}_F \leq \eps''\),
        \(\abs{w_{\pi(i)} -\tilde w_i} \leq \eps'\) and
        \(\norm{\mu_{\pi(i)} - \tilde \mu_i}_2 \leq \eps' 
        \). 
        The expected running time over the randomness of \cref{alg: approx tensor decomposition} is at most $\poly(d,\eps''^{-1},\delta^{-1},\tau,\rho_{max}, \rho_{min}^{-1},w_{min}^{-1},\max_{i\in[d]} \Sigma_{ii}^3)$ and will use $N = \Omega \bigl( \eps''^{-2}\delta^{-1}d^{13}\max_{i\in [d]}\Sigma_{ii}^3 \, \bigl(\poly'_{\ref*{thm:main tensor}}(d,\tau,\rho_{max}, w_{min}^{-1/3}\rho_{min}^{-1})\bigr)^2 \bigr)$ samples.
    \end{theorem}
    \begin{proof}
        Let
        $
            \poly_{\ref*{thm: GMM main theorem}}(d,\rho_{max}) = 1 + d\rho_{max}^2 + 2d(2\rho_{max} + 1)
        $.
        By \cref{thm: mixture learning main theorem}, with probability $1-\delta$, \cref{alg: mixture learning algorithm} will output the estimated  mixing weights $\tilde w_i$ and means $\tilde \mu_i$ within $\eps'$ additive accuracy.
        The sample complexity and running time follows therein, where we have 
        $\max_{i\in [d]}\expectation[X_i^6] = \max_{i\in [d]} 15\Sigma_{ii}^3$ for GMM. 
        
        Next, we bound the error in the covariance matrix. 
        Note that when the number of samples guarantees that $K_3(X)$ is estimated to $\eps_{in}$ accuracy with probability $1-\delta$, it can also guarantee $\cov(X)$ is estimated to $\eps_{in}$ accuracy with probability $1-\delta$
        since the latter takes $\Omega(d^6\eps_{in}^{-2}\delta^{-1}\max_{i\in[d]}\Sigma_{ii}^2)$ many samples by a similar argument to \cref{lem: variance of k3,lem: cumulant sample complexity}.
        So
        \begin{equation*}
            \begin{split}
                \norm{\tilde \Sigma - \Sigma}_F 
                &= \biggnorm{\frac{1}{N}\sum_{j\in[N]}x_jx_j^\top - \sum_{i\in[d]}\tilde w_i \tilde\mu_i\tilde\mu_i^\top - \Sigma}_F \\
                &\leq \biggnorm{\frac{1}{N}\sum_{j\in[N]}x_jx_j^\top - \cov(X)}_F +\sum_{i\in[d]}\abs{w_i-\tilde w_i}\norm{\mu_i\mu_i^\top}_F + \sum_{i\in[d]}\tilde w_i\norm{\mu_i\mu_i^\top-\tilde\mu_i\tilde\mu_i^\top}_F\\
                &\leq \eps_{in} + d\rho_{max}^2\eps' +\sum_{i\in[d]}(w_i + \eps')(2\norm{\mu_i}_2 + \eps')\eps' 
                \leq \poly_{\ref*{thm: GMM main theorem}}\eps' 
                \leq \eps'',
            \end{split}
        \end{equation*}
        where the second to last inequality follows from bounding $\eps_{in}$ by $\eps'$ and $w_i, \eps'$ by 1.
    \end{proof}

\section{Proof of \cref{thm:tensor decomposition main theorem}}\label{sec: proof}

In this section, we implement the three parts mentioned in the ``proof idea'',
in \cref{sec: uniqueness,sec: robustness,sec: probablity bounds section}, respectively.
We combine them in \cref{sec: proof of main theorem}.
        
\subsection{Uniqueness of decomposition}
\label{sec: uniqueness}

We show that if \cref{alg: approx tensor decomposition} satisfies its termination condition, then the outputted components are close to the components of $T$.
We deduce this directly from the following known result on the stability of tensor decompositions.
\begin{theorem}[{\cite[Theorem 2.6]{bhaskara2014uniqueness}}]\label{thm: bhaskara}
    Suppose a rank $R$ tensor $T = \sum_{i\in[R]} a_i^{\otimes3}\in \Real^{d\times d \times d}$ is $\rho$-bounded. Let $A = [a_1,\dotsc,a_R]$ with $3\rkrank{\tau}{A} \geq 2R + 2$. Then for every $\varepsilon'\in(0,1)$, there exists
    $\varepsilon = \varepsilon'/\poly_{\ref*{thm: bhaskara}}(R, \tau, \rho,\rho', d)$
    for a fixed polynomial $\poly_{\ref*{thm: bhaskara}}$ so that for any other $\rho'$-bounded decomposition $T' = \sum_{i\in[R]} (a_i')^{\otimes3}$ with $\norm{T'-T}_F\leq \varepsilon$, there exists a permutation matrix $\Pi$ and diagonal matrix $\Lambda$ such that
    \( 
    \norm{\Lambda^3-I}_F \leq \varepsilon'\) and \(\norm{A' - A\Pi\Lambda}_F \leq \varepsilon'
    \).
\end{theorem}
The original statement in \cite{bhaskara2014uniqueness} explicitly assumes that $T$ (the sum of $R$ rank-1 tensors) has rank $R$, but this assumption is redundant: a tensor $T = \sum_{i\in[R]} a_i^{\otimes3}\in \Real^{d\times d \times d}$ with $3 \krank(A) \geq 2R+2$ cannot have another decomposition with less than $R$ terms because of Kruskal's uniqueness theorem \cite[Theorem 4a]{MR444690}. Note that in \cref{thm: bhaskara} a scaling matrix $\Lambda$ is introduced. We will use the following corollary instead to have a handier result without the scaling matrix:
\begin{corollary}
    \label{col: bhaskara corollary}
   In the setting of \cref{thm: bhaskara}, there exists a polynomial $\poly_{\ref*{col: bhaskara corollary}}(R,\tau,\rho,\rho',d)$ such that if $\varepsilon'\in(0,1)$ and
   $\varepsilon = \varepsilon'/\poly_{\ref*{col: bhaskara corollary}}(R, \tau, \rho,\rho', d)$, 
   then for any other $\rho'$-bounded decomposition $T' = \sum_{i\in[R]} (a_i')^{\otimes3}$ with $\norm{T'-T}_F\leq \varepsilon$, there exists a permutation $\pi$ of $[R]$ such that $\forall i\in[R]$,
    \(
    \norm{a_{\pi(i)} - a'_i}_2\leq \varepsilon'
    \).
\end{corollary}
\begin{proof}
    We assume that the permutation is the identity. 
    Let $c =(1+4\rho/3)$ and $\poly_{\ref*{col: bhaskara corollary}} = c \poly_{\ref*{thm: bhaskara}}$. By \cref{thm: bhaskara}, we have that for each $i\in[R]$:
    \( 
    \norm{a_i' - \lambda_i a_i}_2\leq c^{-1}\varepsilon' \) and \( \abs{\lambda_i^3 -1} \leq c^{-1}\varepsilon'
    \).
    Since $\abs{x-1}\leq 4\abs{x^3-1}$/3 for all $x \in \RR$, the second inequality implies that: 
    \(
        \abs{\lambda_i - 1} \leq 4\abs{\lambda_i^3-1}/3\leq 4c^{-1}\eps'/3 
    \).
    Therefore
            $\norm{a_i' - a_i}_2 \leq \norm{a_i' - \lambda_i a_i}_2 + \abs{\lambda_i - 1}\norm{a_i}_2 \leq  (1 + 4\rho/3)c^{-1}\varepsilon'  = \varepsilon'$.
\end{proof}

\subsection{Robust decomposition}
\label{sec: robustness}
In this subsection, we will derive the forward error propagation of \cref{alg: approx tensor decomposition}, i.e.\ how the output error depends on the input error in each step of \cref{alg: approx tensor decomposition}.
We will assume throughout this subsection that we already have two unit vectors $x,y$ that are nearly orthogonal to $\hat a_{r+1},\ldots, \hat a_{r+k}$, 
that is, 
$ \abs{\inner{x}{\hat a_{r+i}}},\abs{\inner{y}{\hat a_{r+i}}}\leq \theta $ 
for $i\in[k]$, where $\theta$ will be chosen later, and $\rkrank{\tau}{A}\geq r$. 
Let $E_{\mathrm{in}} = T-\tilde T$ be the input error tensor.
\paragraph{Part 1: robust diagonalization} We first cite the robust analysis of \cref{alg: Jennrich}.
\begin{theorem}[{\cite[Theorem 5.4, Lemmas 5.1, 5.2]{goyal2014fourier}}]\label{thm:gvx}
    \label{thm: robustness of Jennrich}
    \hspace{-1ex}Let
    \(
    T_\mu = \sum_{i\in[r]} \mu_i a_i a_i^\top = A\diag(\mu)A^\top \), 
    \(
        T_\lambda = \sum_{i\in[r]} \lambda_i a_i a_i^\top= A\diag(\lambda)A^\top
    \), 
    $A = [a_1,\dotsc,a_r]$, $a_i \in \RR^d$, $\norm{a_i} = 1$, $\lambda_i, \mu_i \in \RR$ for $i \in [r]$.
    Suppose
         \textbf{(1)} $\sigma_r(A) > 0$,
         \textbf{(2)} $(\forall i)$ $0 < k_l \leq \abs{\mu_i}, \abs{\lambda_i} \leq k_u$, and
         \textbf{(3)} $(\forall i \neq j)$ $\abs{\mu_i/\lambda_i - \mu_j/\lambda_j} \geq \alpha > 0$.
    Let  $0<\eps_{\ref*{thm: robustness of Jennrich}}<1$ and $\tilde T_\mu, \tilde T_\lambda$ be matrices such that 
    $
    \fnorm{T_\mu - \tilde T_\mu}, \fnorm{T_\lambda - \tilde T_\lambda} \leq \frac{\eps_{\ref*{thm:gvx}} k_l^2 \sigma_r(A)^3 \min\{\alpha,1\}}{2^{11} \kappa(A) k_u r^2}
    $.
    Then \cref{alg: Jennrich} on input $\tilde T_\mu, \tilde T_\lambda$ outputs unit vectors $\tilde a_1,\dotsc,\tilde a_{r}$ such that for some permutation $\pi$ of $[r]$ and signs $s_1,\dotsc, s_{r} \in \{\pm 1\}$, and for all $i \in [r]$ we have 
    $\norm{a_{\pi(i)} - s_i \tilde a_i} \leq \eps_{\ref*{thm:gvx}}$.
    It runs in time $\poly(d,1/\alpha,1/k_l, 1/\sigma_r(A_r),1/\eps_{\ref*{thm:gvx}})$.
\end{theorem}
Now we apply \cref{thm:gvx} to our case: let $E_x = T_x - \tilde T_x$ and $E_y = T_y - \tilde T_y$. Write
\(
    \tilde T_x = \hat A_{r} D_x \hat A_{r}^\top + \hat A_{>r} D_x'\hat A_{>r}^\top + (E_{\mathrm{in}})_x,
\)
where $\hat A_r$ contains $\hat a_i$s as columns, $D_x = \diag(\norm{a_i}^3\inner{x}{\hat a_i})$ for $i\in[r]$ and $\hat A_{>r}$ contains $\hat a_{r+i}$s, $D_x' = \diag(\norm{a_{r+i}}^3\inner{x}{\hat a_{r+i}})$ for $i\in[k]$. Then we have
\begin{equation}
    \label{equ: norm of Ex}
    \norm{E_x}_F = \norm{\hat A_{>r} D_x'\hat A_{>r}^\top + (E_{\mathrm{in}})_x}_F \leq kM^3\theta + \eps_{in},
\end{equation}
and similarly for $E_y$. The following lemma guarantees the correctness of step \ref{line:decomposition} in \cref{alg: approx tensor decomposition}.
\begin{lemma}
    \label{lem: first decomposition}
    Let $\tilde a_1,\ldots,\tilde a_r$ be the outputs of step \ref{line:decomposition} in \cref{alg: approx tensor decomposition}.
    If
    \textbf{(1)} $\forall i\in[r]$: $0 < \hfrac{k_l}{m^3} \leq \abs{\inner{x}{\hat a_i}}, \abs{\inner{y}{\hat a_i}}\leq 1$, and 
    \textbf{(2)} $\forall i,j\in[r], i\neq j$: $\bigabs{\inner{x}{\hat a_r}/\inner{y}{\hat a_r} - \inner{x}{\hat a_r}/\inner{y}{\hat a_r}} \geq \alpha >0$,
    then there are signs $s_1,\ldots, s_r\in\{\pm1\}$ and a permutation $\pi$ of $[r]$ such that $\forall i\in[r]$: $\norm{\hat a_{\pi(i)} - s_i\tilde a_i} \leq \eps_{\ref*{lem: first decomposition}} 
    := \frac{2^{11}\tau^4M^7r^{5/2}(kM^3\theta + \eps_{in})}{k_l^2\min\{\alpha,1\}}$. 
    This step runs in time $\poly(d, \alpha^{-1}, k_l^{-1}, \tau, M, \eps_{\ref*{lem: first decomposition}}^{-1})$.
\end{lemma}
\begin{proof}
    Condition 1 in \cref{thm: robustness of Jennrich} holds since $\rkrank{\tau}{A} \geq r$: 
    $\sigma_r(\hat A_r) \geq \sigma_r(A_r)/M \geq 1/(\tau M)$. 
    Conditions 2 and 3 in \cref{thm:gvx} hold because of our assumptions. 
    Combining \cref{equ: norm of Ex} and $\rkrank{\tau}{A} \geq r$ which implies
    \(
    \sigma_r(\hat A_r)^3\kappa(\hat A_r)^{-1} = \sigma_r(\hat A_r)^4\sigma_1(\hat A_r)^{-1}\geq  (\sqrt{r}\tau^4M^4)^{-1}
    \),
    the assumptions of \cref{thm:gvx} are satisfied with parameter $k_u = M^3$.
    The claim follows.
\end{proof}
Since $x,y$ are actually chosen at random, we provide the probability for assumptions of \cref{lem: first decomposition} to hold in \cref{sec: probablity bounds section}. 

\paragraph{Part 2: norm estimation} 
The next step is to recover $\norm{a_i}_2$. 
This can be done by solving the least squares problem in step \ref{line:leastsquares}.
To see this, one can verify that when $\tilde a_i = \hat a_i$ and $\tilde T_x = T_x$ (no error in earlier steps),  
$\xi_i = \norm{a_i}_2^3$
is a zero error solution to step \ref{line:leastsquares}.
The following lemma guarantees that we can approximate the norm via step \ref{line:leastsquares}:
\begin{lemma}[Norm estimation]
    \label{lem: norm theorem}
    Let $\tilde b_1,\dotsc, \tilde b_r$ be the columns of $(\tilde A_r^\dagger)^\top$. 
    If \cref{lem: first decomposition} holds with $\eps_{\ref*{lem: first decomposition}} \leq \min\{k_l/(2m^3),(2\sqrt{r}\tau M)^{-1}\}$,
    then $\xi_i = \hfrac{\tilde T(x,\Tilde{b}_i,\Tilde{b}_i)}{\inner{x}{\Tilde{a}_i}}$ for $i\in[r]$
    is the unique solution to step \ref{line:leastsquares} in \cref{alg: approx tensor decomposition} and
    for the permutation $\pi$, signs $s_i$ in \cref{lem: first decomposition} and all $i\in[r]$ we have 
    $    \Abs{\norm{a_{\pi(i)}}_2^3 - s_i{\xi_i}} \leq \eps_{\ref*{lem: norm theorem}} := 2k_l^{-1}m^3 M^2 \left[ 3M\eps_{\ref*{lem: first decomposition}} + rM\eps_{\ref*{lem: first decomposition}}^2 + 4\tau^2(kM^3\theta + \eps_{in}) \right]$.
    \details{
        The number of arithmetic operations to find an exact solution for a rational $d$-by-$r$ linear least squares problem $\min_{x \in \RR^r} \norm{Ax-b}^2$ when $A$ has full column rank is $\poly(d,r)$ by using that the unique solution is given by $A^\dagger b = (A^\top A)^{-1} A^\top b$. 
        This also gives a strongly polynomial time algorithm via Edmonds's Gaussian elimination algorithm.}
\end{lemma}
\begin{proof}
    For simplicity we assume the permutation is the identity. We start by showing $\sigma_r(\tilde A_r) >0$, which implies $\tilde A_r^\dagger \tilde A_r = I_r$ and thus $\tilde b_i$ is orthogonal to $\tilde a_j$ for $i,j\in[r], i\neq j$.
    By \cref{lem: first decomposition}, the distance between corresponding columns of $\tilde A_r\diag(s_i)$ and $\hat A_r$ is at most $\eps_{\ref*{lem: first decomposition}}$. therefore by \cref{thm: weylsvd} we have
    $ \bigabs{\sigma_r \bigl( \tilde A_r\diag(s_i) \bigr) - \sigma_r(\hat A_r)} \leq \norm{\tilde A_r\diag(s_i) - \hat A_r}_2 \leq \sqrt{r}\eps_{\ref*{lem:  first decomposition}}$,
    which implies
        \begin{equation}
        \label{equ: least singular value of tilde A_r}
        \sigma_r(\tilde A_r) 
        = \sigma_r \bigl(\tilde A_r\diag(s_i) \bigr) 
        \geq \sigma_r(\hat A_r) - \sqrt{r}\eps_{\ref*{lem:  first decomposition}} 
        \geq (\tau M)^{-1}  - \sqrt{r}\eps_{\ref*{lem:  first decomposition}} \geq \hfrac{1}{(2\tau M)}.
    \end{equation}

    Next, we show that $\xi_i$ is the unique solution to step \ref{line:leastsquares}.
    We restate the least squares problem in a matrix-vector product form,
    \( 
        \min_{\xi_i} \norm{\tilde A^{\odot2}[\inner{x}{\tilde a_1}\xi_1,\ldots,\inner{x}{\tilde a_r}\xi_r]^\top - \vecop(\tilde T_x) }_2
    \),
    where $\tilde A^{\odot2} = [\vecop(\tilde a_1\tilde a_1^\top),\dotsc, \vecop(\tilde a_r\tilde a_r^\top)] \in \Real^{d^2\times r}$. It follows that $\sigma_r(\tilde A^{\odot2}) = \sigma_r(\tilde A_r)^2 > 0$ and thus the solution is unique. Let $\tilde B^{\odot2} = [\vecop(\tilde b_1\tilde b_1^\top),\dotsc, \vecop(\tilde b_r\tilde b_r^\top)]^\top$ and notice that $\tilde B^{\odot2}\tilde A^{\odot2} = I_r$. 
    The solution to the least squares problem is then given by
    \( 
       [\inner{x}{\tilde a_1}\xi_1,\ldots,\inner{x}{\tilde a_r}\xi_r]^\top = \tilde B^{\odot2} \vecop(\tilde T_x) = [\tilde b_1^\top \tilde T_x \tilde b_1,\ldots,\tilde b_r^\top \tilde T_x \tilde b_r]^\top
    \),
    which implies $\xi_i = \tilde T(x,\tilde b_i,\tilde b_i)/\inner{x}{\tilde a_i}$.
    
    Finally we show that $s_i\xi_i$ is close to $\norm{a_i}_2^3$. 
    The deviation of $s_i\xi_i$ from $\norm{a_i}^3_2$ is bounded by:
    \begin{equation}
        \label{equ: multiplier error}
        \begin{split}
        \Abs{\norm{a_i}_2^3 &- s_i\xi_i} 
        = \biggl\lvert \norm{a_i}^3_2 - \frac{1}{\inner{x}{s_i\Tilde{a}_i}}\Bigl(\sum_{j\in[r]}\inner{x}{a_j}\inner{\tilde b_i}{a_j}^2 + \tilde b_j^\top E_x \tilde b_j\Bigr) \biggr\rvert \\
        &\leq \biggl\lvert \frac{\inner{x}{\Hat{a}_i}\inner{\Tilde{b}_i}{\Hat{a}_i}^2}{\inner{x}{s_i\Tilde{a}_i}} - 1 \biggr\rvert \norm{a_i}_2^3 
            + \sum_{j \in [r], j\neq i} \left(\norm{a_j}_2^3\biggl\lvert \frac{\inner{x}{\hat{a}_j}\inner{\Tilde{b}_i}{\hat{a}_j}^2}{\inner{x}{s_i\Tilde{a}_i}}\biggr\rvert + \biggl\lvert \frac{\tilde b_i^\top E_x \tilde b_i}{\inner{x}{s_i\Tilde{a}_{i}}} \biggr\rvert \right).
        \end{split}
    \end{equation}
    We analyze the deviation of each term in \cref{equ: multiplier error}. By standard arguments using triangle and Cauchy-Schwarz inequalities, we have for all $i, j\in[r]$:
    \begin{equation}
        \label{equ: inner product bounds}
        \begin{split}
            &\abs{\inner{x}{s_i\Tilde{a}_i}} \geq \abs{\inner{x}{\Hat{a}_i}}-\eps_{\ref*{lem:  first decomposition}} \geq k_l/m^3-\eps_{\ref*{lem:  first decomposition}} \geq k_l/(2m^3),\\
            &\abs{\inner{x}{s_j\Tilde{a}_j} - \inner{x}{\Hat{a}_j}}\leq \eps_{\ref*{lem:  first decomposition}},\quad \abs{\inner{\tilde{b}_i}{s_j\tilde{a}_j}-\inner{\Tilde{b}_i}{\hat{a}_j}} \leq \eps_{\ref*{lem:  first decomposition}},
        \end{split}
    \end{equation}
    where the first line comes from the assumptions of the lemma, and the last line follows from \cref{lem: first decomposition}.
    Notice that $\tilde b_i$ is orthogonal to $\tilde a_j$ for $j\neq i$, and $\inner{\tilde b_i}{\tilde a_i} = 1$. 
    \Cref{equ: inner product bounds} implies that:
    \begin{equation}
        \label{equ: multiplier error term 1}
        \biggabs{\frac{\inner{x}{\Hat{a}_i}\inner{\Tilde{b}_i}{\Hat{a}_i}^2}{\inner{x}{s_i\Tilde{a}_i}} -1 } \leq 6k_l^{-1}m^3\eps_{\ref*{lem:  first decomposition}},\quad 
        \biggabs{\frac{\inner{x}{\hat{a}_j}\inner{\Tilde{b}_i}{\hat{a}_j}^2}{\inner{x}{s_i\Tilde{a}_i}}}\leq 2k_l^{-1}m^3\eps_{\ref*{lem:  first decomposition}}^2.
    \end{equation}
    The last term in \cref{equ: multiplier error} is bounded by:
    \begin{equation}
        \label{equ: multiplier error term 2}
        \biggabs{\frac{\tilde b_i^\top E_x \tilde b_i}{\inner{x}{s_i\Tilde{a}_{i}}}} 
        \leq 2k_l^{-1}m^3 \norm{E_x}_2 \norm{\tilde b_i}_2^2 
        \leq 2k_l^{-1}m^3\norm{E_x}_F\sigma_r(\tilde A_r)^{-2} 
        \leq 8k_l^{-1}m^3 \tau^2M^2\norm{E_x}_F,
    \end{equation}
    where the second inequality follows from the definition of $\tilde b_i$, and the last inequality applies \cref{equ: least singular value of tilde A_r}. Combining \cref{equ: norm of Ex,equ: multiplier error,equ: multiplier error term 1,equ: multiplier error term 2} gives the desired result.
\end{proof}

\paragraph{Part 3: deflation} 
After we deflate $T$ with the previously recovered $r$ components, the induced error with respect to the exact deflation $\sum_{i=r+1}^{r+k} a_i^{\otimes3}$ is given by\details{actual deflated $ T + E_{\mathrm{in}} - \sum_{i\in[r]}  \xi_i\Tilde{a}_i^{\otimes3}$ minus ideal deflated: $\sum_{i=r+1}^{r+k} a_i^{\otimes3} $} 
$
E' = E_{\mathrm{in}} + \sum_{i\in[r]} (a_i^{\otimes3} - \xi_i\Tilde{a}_i^{\otimes3})
$. 
Now we show that the remaining tensor can be decomposed with the same strategy via step \ref{line:decomposition2} in \cref{alg: approx tensor decomposition}.
\begin{lemma}[Deflation]
    \label{lem: deflation}
    Let $\Tilde{a}_{r+1},\dotsc, \tilde a_{r+k}$ be the outputs of step \ref{line:decomposition2} in \cref{alg: approx tensor decomposition}.
    If
        \textbf{(1)} $\forall i \in[k]$: $0<\hfrac{k_l'}{m^3}\leq \bigabs{\inner{x'}{\hat a_{r+i}}}, \bigabs{\inner{y'}{\hat a_{r+i}}}\leq 1$, and
        \textbf{(2)} $\forall i,j\in[k], i\neq j$: $\bigabs{\inner{x'}{\hat a_{r+i}}/\inner{y'}{\hat a_{r+i}} - \inner{x'}{\hat a_{r+j}}/\inner{y'}{\hat a_{r+j}}} \geq \alpha' >0$,
    then there are signs $s_{r+1},\ldots, s_{r+k}\in\{\pm 1\}$ and a permutation $\pi'$ of $[k]$ such that $\forall i\in[k]$:
    \(
        \norm{\Hat{a}_{r+\pi'(i)} - s_{r+i}\Tilde{a}_{r+i}}_2\leq \eps_{\ref*{lem: deflation}} :=  \frac{2^{11}\tau^4M^7k^{5/2} \norm{E'}_F }{(k'_l)^2\min\{\alpha',1\}}
     \). This step runs in time $\poly(d, {k'_l}^{-1}, {\alpha'}^{-1}, \tau, M, \eps_{\ref*{lem: deflation}}^{-1})$.
\end{lemma}
\begin{proof}
    The proof is similar to the proof of \cref{lem: first decomposition} and thus omitted here.
    \details{
    We only need to show that \cref{thm: robustness of Jennrich} can be applied here. Take 
    \[
    \begin{array}{ll}
        \tilde T_\mu = R_{x'} & T_\mu = \hat A_{>r} \diag(\norm{a_{r+i}}_2^3\inner{x'}{\hat a_{r+i}}) \hat A_{>r}^\top \\
        \tilde T_{\lambda} = R_{y'} & T_\lambda = \hat A_{>r} \diag(\norm{a_{r+i}}_2^3\inner{y'}{\hat a_{r+i}}) \hat A_{>r}^\top.
    \end{array}\]
    Condition 1 in \cref{thm: robustness of Jennrich} holds because $\rkrank{\tau}{A} \geq r$.
    Condition 2 and 3 in \cref{thm: robustness of Jennrich} follow with parameters $k_l', M^3,\alpha'$. 
    Combining
    \( 
        \norm{\tilde T_\mu - T_\mu}_F, 
        \norm{\tilde T_\lambda- T_\lambda}_F 
        \leq \norm{E'}_F
    \)
    with the robust Kruskal rank condition $\rkrank{\tau}{A} \geq r$ which guarantees
    \(\sigma_k(\hat A_{>r})^3\kappa(\hat A_{>r})^{-1} = \sigma_k(\hat A_{>r})^4\sigma_1(\hat A_{>r})^{-1} (\sqrt{k}\tau^4M^4)^{-1},\)
    \cref{thm: robustness of Jennrich} hold and the claim follows.
    }    
\end{proof}
With $\tilde a_{r+1},\dotsc,\tilde a_{r+k}$, we can further approximate the norm of $a_{r+1},\ldots,a_{r+k}$, in the same way we did for the first $r$ components, via step \ref{line:leastsquares2}. 
The following lemma guarantees it works:
\begin{lemma}
    \label{lem: norm theorem 2}
    Let $\tilde b_{r+1}, \dotsc, \tilde b_{r+k}$ be the columns of $(\tilde A_{>r}^\dagger)^\top$. 
    If \cref{lem: deflation} holds with $\eps_{\ref*{lem: deflation}} \leq \min\{k_l'/(2m^3),(2\sqrt{k}\tau M)^{-1}\}$, then
    $\xi_{r+i} = \hfrac{R(x',\Tilde{b}_{r+i},\Tilde{b}_{r+i})}{\inner{x'}{\Tilde{a}_{r+i}}}$, for $i\in[k]$
    is the unique solution to step \ref{line:leastsquares2} in \cref{alg: approx tensor decomposition} and
    for the permutation $\pi'$, signs $s_{r+i}$ in \cref{lem: deflation}, and all $i\in[k]$ we have
    $
        \Abs{\norm{a_{r+\pi'(i)}}^3 - s_{r+i}\xi_{r+i}} 
        \leq \eps_{\ref*{lem: norm theorem 2}} 
        := 2k_l'^{-1} m^3 M^2 \bigl[ 3M\eps_{\ref*{lem: deflation}}+kM\eps_{\ref*{lem: deflation}}^2
        + 4 \tau^2\norm{E'}_F \bigr]
    $.
\end{lemma}
\begin{proof}
    The proof is similar to the proof of \cref{lem: norm theorem} and thus omitted here.
    \details{We start by bounding $\sigma_k(\tilde A_{>r})$ from below. By \cref{lem: deflation}, the distance between corresponding columns of $\tilde A_r\diag(s_i)$ and $\hat A_r$ is at most $\eps_{\ref*{lem: deflation}}$. Similarly, by \cref{thm: weylsvd}:
    \begin{equation}
        \label{equ: least singular value of tilde A>r}
        \sigma_k(\tilde A_{>r}) = \sigma_k(\diag(s_{r+i})\tilde A_{>r}) \geq \sigma_k(\hat A_{>r}) - \sqrt{k}\eps_{\ref*{lem: deflation}} \geq \frac{1}{\tau M} - \sqrt{k}\eps_{\ref*{lem: deflation}} \geq \frac{1}{2\tau M}.
    \end{equation}
    Thus by reformulating step \ref{line:leastsquares2}, we can show that $\xi_{r+i} = \hfrac{R(x',\Tilde{b}_{r+i},\Tilde{b}_{r+i})}{\inner{x'}{\Tilde{a}_{r+i}}}$ is the unique solution to the least squares problem
    and $\Abs{\norm{a_{r+i}}^3 - s_{r+i}\xi_{r+i}}$ is bounded by:
    \begin{equation}
        \label{equ: deflated norm estimation}
        \begin{split}
            \Abs{\norm{a_{r+i}}^3 - s_{r+i}\xi_{r+i}} &\leq \norm{a_{r+i}}^3\Abs{\frac{\inner{x'}{\Hat{a}_{r+i}}\inner{\Tilde{b}_{r+i}}{\Hat{a}_{r+i}}^2}{\inner{x'}{s_{r+i}\Tilde{a}_{r+i}}} - 1}\\ 
            &\quad + \sum_{j\in[k] \setminus i} \bigl(\norm{a_{r+j}}^3 \Abs{\frac{\inner{x'}{\Hat{a}_{r+j}}\inner{\Tilde{b}_{r+i}}{\Hat{a}_{r+j}}^2}{\inner{x'}{s_{r+i}\Tilde{a}_{r+i}}}} + \Abs{\frac{\tilde b_{r+i} ^\top E'_x \tilde b_{r+i}}{\inner{x'}{s_{r+i}\Tilde{a}_{r+i}}}}\bigr).
        \end{split}
    \end{equation}
    Similar to \cref{equ: inner product bounds}, we have the following bounds for the terms in \cref{equ: deflated norm estimation}:
    \begin{equation}
        \label{equ: inner product bounds 2}
        \begin{split}
            &\abs{\inner{x'}{s_{r+i}\tilde a_{r+i}}} \geq k_l'/m^3 -\eps_{\ref*{lem: deflation}} \geq k'_l/(2m^3) \quad \text{for }i\in[k], \\
            &\abs{\inner{x'}{\Hat{a}_{r+j}} - \inner{x'}{s_{r+i}\Tilde{a}_{r+j}}}\leq \eps_{\ref*{lem: deflation}},\quad \abs{\inner{\Tilde{b}_{r+i}}{\Hat{a}_{r+j}} - \inner{\Tilde{b}_{r+i}}{s_{r+i}\Tilde{a}_{r+j}}} \leq \eps_{\ref*{lem: deflation}} \quad \text{for }i\in[k],
        \end{split}
    \end{equation}
    where the first inequality is from assumptions of the lemma, the second and the last are from the conclusion of \cref{lem: deflation}.
    Now we can bound \cref{equ: deflated norm estimation} with \cref{equ: least singular value of tilde A>r,equ: inner product bounds 2}:
    \begin{equation*}
        \Abs{\norm{a_{r+i}}^3 - s_{r+i}\xi_{r+i}}\leq 2{k'_l}^{-1}m^3M^3(3\eps_{\ref*{lem: deflation}}+(k-1)\eps_{\ref*{lem: deflation}}^2) + 8{k'_l}^{-1}m^3\tau^2M^2\norm{E'}_F.
    \end{equation*}}
\end{proof}

\subsection{Probability bounds}
\label{sec: probablity bounds section}
We give here bounds on the probability of finding good  $x,y,x',y'$ so that \cref{alg: approx tensor decomposition} succeeds with positive probability. 
Throughout this subsection, let $x,y$ be two iid.\ random vectors distributed uniformly on $\sphere^{d-1}$, and $\rkrank{\tau M}{[\hat a_1,\dotsc,\hat a_{r+k}]}\geq r$.

We first list the events that need to hold to apply \cref{lem: first decomposition}:
\begin{enumerate}
    \item vanishing last $k$ terms: $\event_{1,y}=\{\forall i \in [k], \abs{\inner{y}{\hat a_{r+i}}} \leq \theta\}$;
    \item lower bounds on first $r$ terms: $\event_{2,y} = \{\forall i\in[r], \abs{\inner{y}{\hat a_i}} \geq k_l/m^3\}$;
    \item the eigenvalue gap: $\event_3=\{\forall i\neq j, i,j\in[r],\abs{\inner{x}{\hat a_i}/\inner{y}{\hat a_i} - \inner{x}{\hat a_j}/\inner{y}{\hat a_j}} \geq \alpha >0\}$.
\end{enumerate}
We have similar events $\event_{1,x},\event_{2,x}$. 
Note that in this subsection $k_l,\theta$ and $\alpha$ are considered as fixed parameters.

The structure of this subsection is stated as follows: 
we will first demonstrate our proof idea for controlling the probability of the listed events, as the union bound would be too weak to work for them. 
After presenting our idea, we will first analyze the probability of $\event_{1,y}\cap \event_{2,y}$, then the probability of $\event_{1,x}\cap\event_{2,x}\cap\event_3$ when conditioned on the other events of $y$. 
Finally we will collect these sub-events and give the probability that all of them will hold.


To bound the probability of $\event_{1,y}\cap\event_{2,y}$, we give the idea of our analysis below:
\details{note that for our algorithm to succeed, $\theta$ needs to be close to zero and much smaller than $k_l$, and in fact $\event_{1,y}$ happens with somewhat small probability. 
Hence the union bound is too weak to work here. 
We need to carefully bound the probability for these events.}

\paragraph{Bands argument}
We analyze the events geometrically and replace random unit vectors by random Gaussian vectors together with concentration of their norm. 
Let $z$ be a random Gaussian vector let $a$ and $b$ be two unit vectors.
An event of the form $\{\abs{\inner{z}{a}} \leq t_1\}$ corresponds to a band, while an event like $\{\abs{\inner{z}{b}} \geq t_2\}$ corresponds to the complement of a band. 
We call them bands of type \RN{1} and type \RN{2}, denoted by $\mathcal{B}_{1}$ and $\mathcal{B}_{2}$, respectively. 
To better illustrate this, we give a demonstration of bands as the shaded areas in \cref{fig:bands}. 
\begin{figure}[htbp]
    \centering
    \subfloat[Band of type \RN{1}]{\label{fig: band 1}\tikzset{every picture/.style={line width=0.75pt}} 
        \begin{tikzpicture}[x=0.75pt,y=0.75pt,yscale=-.75,xscale=.75]
        \draw [color={rgb, 255:red, 0; green, 0; blue, 0 }  ,draw opacity=1 ][fill={rgb, 255:red, 137; green, 102; blue, 102 }  ,fill opacity=1 ]   (428.25,119.75) -- (198.78,119.75) ;
        \draw [color={rgb, 255:red, 0; green, 0; blue, 0 }  ,draw opacity=1 ][fill={rgb, 255:red, 137; green, 102; blue, 102 }  ,fill opacity=1 ]   (428.25,170.75) -- (198.78,170.75) ;
        
        \draw  [draw opacity=0][fill={rgb, 255:red, 158; green, 158; blue, 158 }  ,fill opacity=1 ] (198.78,119.75) -- (428.25,119.75) -- (428.25,170.75) -- (198.78,170.75) -- cycle ;
        
        \draw (304,136) node [anchor=north west][inner sep=0.75pt]  [xscale=0.8,yscale=0.8]  {$\mathcal{B}_{1}$};

    \end{tikzpicture}}\hspace{3cm}
    \subfloat[Band of type \RN{2}]{\label{fig: band 2}\tikzset{every picture/.style={line width=0.75pt}} 

        \begin{tikzpicture}[x=0.75pt,y=0.75pt,yscale=-1,xscale=1]
        
        \draw  [draw opacity=0][fill={rgb, 255:red, 255; green, 255; blue, 255 }  ,fill opacity=1 ] (193,137) -- (387.5,137) -- (387.5,160.78) -- (193,160.78) -- cycle ;
        \draw  [draw opacity=0][fill={rgb, 255:red, 128; green, 128; blue, 128 }  ,fill opacity=1 ] (193,124.11) -- (387.5,124.11) -- (387.5,137) -- (193,137) -- cycle ;
        \draw    (193,137) -- (387.5,137) ;
        \draw    (193,160) -- (387.5,160) ;
        \draw  [draw opacity=0][fill={rgb, 255:red, 128; green, 128; blue, 128 }  ,fill opacity=1 ] (193,160) -- (387.5,160.78) -- (387.5,173.67) -- (193,173.67) -- cycle ;
        
        \draw (277,142) node [anchor=north west][inner sep=0.75pt]  [font=\normalsize,xscale=0.8,yscale=0.8] [align=left] {$\displaystyle \mathcal{B}_{2}$};
        \end{tikzpicture}}
    \caption{Example of bands}\label{fig:bands}
\end{figure}
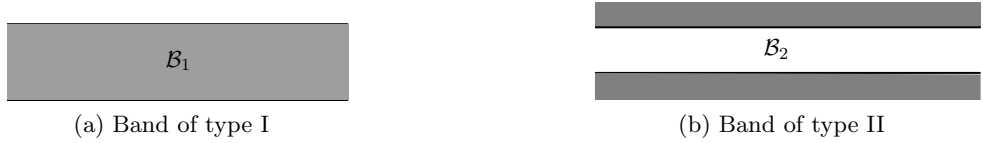
The intersection of bands of type \RN{1} can be lower-bounded with \cref{lem: uniform distribution lower bound 2} (a direct use of the Gaussian correlation inequality), while the intersection of bands of different types needs special care. 
Consider $\mathcal{B}_1\cap \mathcal{B}_2$: when $\inner{a}{b}=0$, the intersection becomes $\mathcal{B}_{1}$ with a rectangular region excluded. 
In this case, the two bands will be orthogonal, and the two events are independent.
In the general case, the excluded region is a parallelogram depending on $\inner{a}{b}$. See \cref{fig:intersection} for illustration. 
In the extreme case, two bands are parallel and hence the probability will be zero when $t_1 \leq t_2$. But when $\inner{a}{b}$ is not too close to one, we can, when bounding the probability, replace the parallelogram by a slightly larger rectangular region without decreasing the final probability too much, which is shown by the white dashed lines in \cref{fig: intersection 2}. This is essentially done by projecting $b$ onto $\spanop{\{a\}}$ and $\spanop{\{a\}}^\perp$.
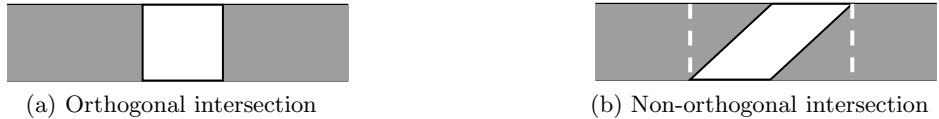
\begin{figure}[htbp]
    \centering
    \subfloat[Orthogonal intersection]{\label{fig: intersection 1}\tikzset{every picture/.style={line width=0.75pt}} 
    
    \begin{tikzpicture}[x=0.75pt,y=0.75pt,yscale=-0.75,xscale=0.75]
    
    \draw [color={rgb, 255:red, 0; green, 0; blue, 0 }  ,draw opacity=1 ][fill={rgb, 255:red, 137; green, 102; blue, 102 }  ,fill opacity=1 ]   (428.25,119.75) -- (198.78,119.75) ;
    \draw [color={rgb, 255:red, 0; green, 0; blue, 0 }  ,draw opacity=1 ][fill={rgb, 255:red, 137; green, 102; blue, 102 }  ,fill opacity=1 ]   (428.25,170.75) -- (198.78,170.75) ;
    
    \draw  [draw opacity=0][fill={rgb, 255:red, 158; green, 158; blue, 158 }  ,fill opacity=1 ][line width=0.75]  (198.78,119.75) -- (428.25,119.75) -- (428.25,170.75) -- (198.78,170.75) -- cycle ;
    \draw  [color={rgb, 255:red, 0; green, 0; blue, 0 }  ,draw opacity=1 ][fill={rgb, 255:red, 255; green, 255; blue, 255 }  ,fill opacity=1 ] (290,119.75) -- (344,119.75) -- (344,170.75) -- (290,170.75) -- cycle ;
    
    \end{tikzpicture}}
    \hspace{3cm}
    \subfloat[Non-orthogonal intersection]{\label{fig: intersection 2}\tikzset{every picture/.style={line width=0.75pt}} 

    \begin{tikzpicture}[x=0.75pt,y=0.75pt,yscale=-0.75,xscale=0.75]
    
    \draw [color={rgb, 255:red, 0; green, 0; blue, 0 }  ,draw opacity=1 ][fill={rgb, 255:red, 137; green, 102; blue, 102 }  ,fill opacity=1 ]   (428.25,119.75) -- (198.78,119.75) ;
    \draw [color={rgb, 255:red, 0; green, 0; blue, 0 }  ,draw opacity=1 ][fill={rgb, 255:red, 137; green, 102; blue, 102 }  ,fill opacity=1 ]   (428.25,170.75) -- (198.78,170.75) ;
    
    \draw  [draw opacity=0][fill={rgb, 255:red, 158; green, 158; blue, 158 }  ,fill opacity=1 ][line width=0.75]  (198.78,119.75) -- (428.25,119.75) -- (428.25,170.75) -- (198.78,170.75) -- cycle ;
    \draw  [color={rgb, 255:red, 0; green, 0; blue, 0 }  ,draw opacity=1 ][fill={rgb, 255:red, 255; green, 255; blue, 255 }  ,fill opacity=1 ] (317.5,119.75) -- (371.5,119.75) -- (316.5,170.75) -- (262.5,170.75) -- cycle ;
    \draw [color={rgb, 255:red, 255; green, 255; blue, 255 }  ,draw opacity=1 ][line width=1.5]  [dash pattern={on 5.63pt off 4.5pt}]  (262.5,119.83) -- (262.5,170.75) ;
    \draw [color={rgb, 255:red, 255; green, 255; blue, 255 }  ,draw opacity=1 ][line width=1.5]  [dash pattern={on 5.63pt off 4.5pt}]  (371.5,119.75) -- (371.5,170.67) ;
    \end{tikzpicture}}
    \caption{Intersection of bands}\label{fig:intersection}
\end{figure}

We see that events $\event_{1,y}, \event_{2,y}$ are the intersection of bands and their probability is the probability measure of their intersection. 
Specifically, we have:
$\event_{1,y} = \cap^k_{i=i}\mathcal{B}_{1,i},$ $\event_{2,y} = \cap^r_{j=1}\mathcal{B}_{2,j},$
where $\mathcal{B}_{1,i}:= \{\abs{\inner{y}{\hat a_{r+i}}}\leq \theta\}$ and $\mathcal{B}_{2,j}:= \{\abs{\inner{y}{\hat a_j}}\geq k_l/m^3\}$. For the rest of this subsection, let $S = \spanop\{\hat a_{r+1},\dotsc,\hat a_{r+k}\}^\perp$, $S^\perp = \spanop\{\hat a_{r+1},\dotsc,\hat a_{r+k}\}$, and $\proj_S$ be the orthogonal projection onto $S$ and $\proj_{S^\perp} = I -\proj_S$.
Now we can bound the probability of $\event_{1,y}\cap\event_{2,y}$:
\begin{lemma}
    \label{lem: tensor probability bound 2}
    If $k_l>0$ and $0 < \theta\leq 2/\sqrt{d}$, then
    $\prob[\event_{1,y}\cap \event_{2,y}] 
    \geq p_1:= (\theta \sqrt{d}/8)^k \bigl(1/4- r\sqrt{d/2\pi}\tau M (4k_l/m^3 + \sqrt{k}\tau M\theta) \bigr)$.
\end{lemma}
\begin{proof}
Write $y = z/\enorm{z}$, where $z$ is a standard Gaussian random vector.
Consider the following events corresponding to $z$, for $\lowerr, \upperr$ to be chosen later: 
$\mathcal{B}'_{1,i} := \{\abs{\inner{z}{\hat a_{r+i}}}\leq \lowerr \theta\}$ and 
$\mathcal{B}'_{2,j} := \{\abs{\inner{z}{\hat a_j}} \geq \upperr k_l/m^3\}$. 
We have
\begin{align*}
\event_{1,y} \cap \event_{2,y} 
&= (\cap_i \ev{B}_{1,i}) \cap (\cap_j \ev{B}_{2,j})
= (\cap_i \ev{B}_{1,i}) \setminus (\cup_j \ev{B}_{2,j}^c) \\
&\supseteq (\cap_i \ev{B}'_{1,i} \setminus \{ \enorm{z} \leq \lowerr \}) \setminus \cup_j ((\ev{B}'_{2,j})^c \cup \{\enorm{z} \geq \upperr \}).
\end{align*}
Set $\eventA = \cap_{i\in[k]} \ev{B}'_{1,i}$.
Since $\eventA \setminus \{ \enorm{z} \leq \lowerr\}
= \eventA \setminus (\{ \enorm{z} \leq \lowerr\} \cap \eventA) 
\supseteq \eventA \setminus (\{ \enorm{\proj_S z} \leq \lowerr \} \cap \eventA)
$:
\begin{align}
\event_{1,y} \cap \event_{2,y} 
&\supseteq \bigl(\eventA \setminus (\{ \enorm{\proj_S z} \leq \lowerr \} \cap \eventA) \bigr) \setminus \cup_j ((\ev{B}'_{2,j})^c \cup \{\enorm{z} \geq \upperr \}) \nonumber \\ 
&= \eventA \setminus \Bigl( 
    (\{ \enorm{\proj_S z} \leq \lowerr \} \cap \eventA) 
    \bigcup \cup_{j \in [r]} ((\ev{B}'_{2,j})^c \cap \eventA) 
    \bigcup (\{\enorm{z} \geq \upperr\} \cap \eventA) 
    \Bigr)\label{equ:terms}
\end{align}
We now bound the probabilities of the terms in \eqref{equ:terms}. First,
\begin{align*}
    \pr[(\ev{B}'_{2,j})^c, \eventA] &= \prob \bigl[ \abs{\inner{z}{\hat a_j}} \leq \upperr k_l/m^3 \bigm| \eventA \bigr] \prob[\eventA].
\end{align*}
Notice that when conditioning on the event $\abs{\inner{z}{\hat a_{r+i}}}\leq \lowerr \theta$ for $i\in[k]$ we have:
\begin{equation}
    \label{equ: projection after condition2}
    \abs{\inner{z}{\proj_{S^\perp}\hat a_j}} = \abs{z^\top \hat A_{>r} \hat A_{>r}^\dagger \hat a_j} \leq \lowerr \sqrt{k}\theta\norm{\hat A_{>r}^\dagger\hat a_j}_2 \leq \lowerr \sqrt{k}\tau M\theta,
\end{equation}
where the first equality comes from the definition of the projection, the second inequality follows from the conditioning, and the last comes from the robust Kruskal rank condition. 
Furthermore, we notice that $\proj_S\hat a_j$ is orthogonal to $\hat a_{r+1},\dotsc,\hat a_{r+k}$ and the conditioning can therefore be dropped after applying \cref{equ: projection after condition2}:
\begin{equation*}
    \begin{split}
        \prob\bigl[\abs{\inner{z}{\hat a_j}} \leq \upperr k_l/m^3 \bigm| \eventA \bigr]
        &\leq \prob\bigl[\abs{\inner{z}{\proj_S \hat a_j}} \leq \upperr k_l/m^3 + \abs{\inner{z}{\proj_{S^\perp} \hat a_1}} \bigm| \eventA \bigr] \\
        &\leq \prob\bigl[\abs{\inner{z}{\proj_S \hat a_j}} \leq \upperr k_l/m^3 + \lowerr \sqrt{k}\tau M\theta\bigr]\\
        &\leq 2(\sqrt{2\pi} \enorm{\proj_S\hat a_1})^{-1} (\upperr k_l/m^3 + \lowerr \sqrt{k}\tau M\theta)\\
        &\leq \sqrt{2/\pi} \tau M (\upperr k_l/m^3 + \lowerr \sqrt{k}\tau M\theta),
    \end{split}
\end{equation*}
where the last two steps follow from bounding the density of a Gaussian distribution from above and the fact that $\{\hat a_j, \hat a_{r+1},\dotsc,\hat a_{r+k}\}$ also satisfies the robust Kruskal rank condition so that $\norm{\proj_S\hat a_j}_2 \geq (\tau M)^{-1}$.

We use the following bounds for the rest of the terms in \eqref{equ:terms}:
\begin{align*}
\pr[\eventA] &\geq (\lowerr \theta/4)^k \quad \text{(\cref{lem: uniform distribution lower bound 2})},\\
\pr[\enorm{\proj_S z} \leq \lowerr, \eventA] 
    &= \pr[\enorm{\proj_S z} \leq \lowerr] \pr[\eventA] \leq \pr[\eventA]/2 \quad \text{(set $\lowerr = \sqrt{d}/2$)},\\
\pr[\enorm{z} \geq \upperr, \eventA] 
&= \pr[\enorm{z} \geq \upperr \giventhat \eventA] \pr[\eventA] 
= (1-\pr[\enorm{z} \leq \upperr \giventhat \eventA]) \pr[\eventA] \\
&\leq (1-\pr[\enorm{z} \leq \upperr]) \pr[\eventA] \qquad \text{(Gaussian correlation inequality), and} \\
&\leq \pr[\eventA]/4 \qquad \text{(Markov's inequality, set $\upperr = 2\sqrt{d}$)}.
\end{align*}
Combining the previous estimates we get
$\pr[\event_{1,y} \cap \event_{2,y}] 
\geq \pr[\eventA]\bigl(1 - 1/2 - r \sqrt{2/\pi} \tau M ( \upperr k_l/m^3 + \lowerr \sqrt{k}\tau M\theta) - 1/4 \bigr)$.
The claim follows.
\end{proof}

At this point, we are ready to analyze the probability of $\mathcal{E}_3$.
\begin{lemma}
    \label{lem: tensor probability bound 3}
    In the setting of \cref{lem: tensor probability bound 2}, let
    \(
    p_2 = p_1 - (\theta\sqrt{d}/8)^k r^2\tau M(\sqrt{dk}\theta\tau M k_l^{-1}m^3+ \alpha)
    \).
    Then
    $
    \prob[\event_3\cap\event_{1,x}\cap\event_{2,x} \vert \event_{1,y},\event_{2,y}]\geq p_2.
    $
\end{lemma}
\begin{proof}
    We start with our idea to bound the probability of the ``eigenvalue gap" $\Bigabs{\frac{\inner{x}{\hat a_s}}{\inner{y}{\hat a_s}} - \frac{\inner{x}{\hat a_t}}{\inner{y}{\hat a_t}}}\geq \alpha$ for $s,t\in[r], s\neq t$.
    Since we condition on $\abs{\inner{y}{\hat a_i}}$ not being too small for all $i\in[r]$, when further conditioned on $y$, we have:
    \begin{equation*}
        \begin{split}
            \prob \biggl[\Bigabs{\frac{\inner{x}{\hat a_s}}{\inner{y}{\hat a_s}} - \frac{\inner{x}{\hat a_t}}{\inner{y}{\hat a_t}}}\geq \alpha \biggm| \mathcal{E}_{1,y},\mathcal{E}_{2,y} \biggr] =\expectation\biggl[\prob\biggl[\Bigabs{\frac{\inner{x}{\hat a_s}}{\inner{y}{\hat a_s}} - \frac{\inner{x}{\hat a_t}}{\inner{y}{\hat a_t}}}\geq \alpha \biggm| y\biggr] \biggm| \mathcal{E}_{1,y},\mathcal{E}_{2,y}\biggr]. \\
        \end{split}
    \end{equation*}
    Therefore it is enough to show a uniform lower bound for $\prob[\abs{\inner{x}{C_s \hat a_s - C_t \hat a_t}}\geq \alpha]$, where $\abs{C_s},\abs{C_t}$ are in $[1, k_l^{-1}m^3]$. 
    We notice that the set $\{\abs{ \inner{x}{C_s \hat a_s - C_t \hat a_t}}\geq \alpha\}$ generates a type \RN{2} band, denoted by $\mathcal{B}_{3,st}$. 
    Therefore the target event is the intersection of $k$ type \RN{1} bands $\mathcal{B}_{1,i}$, $r$ type \RN{2} bands $\mathcal{B}_{2,j}$ and $\binom{r}{2}$ type \RN{2} bands $\mathcal{B}_{3,st}$.     
    More precisely,
    \[
    \prob[\event_3, \event_{1,x}, \event_{2,x} \mid \event_{1,y},\event_{2,y}] 
    \geq \inf_{\abs{C_s},\abs{C_t} \in [1, k_l^{-1}m^3]} 
    \prob[\cap_{i\in[k]}\mathcal{B}_{1,i} , 
    \cap_{j\in[r]}{\mathcal{B}_{2,j}} , 
    \cap_{s,t\in[r],s\neq t} \mathcal{B}_{3,st} ].
    \]
    We reuse ideas from the proof of \cref{lem: tensor probability bound 2}. Write $x = u/\norm{u}_2$ with $u$ being standard Gaussian. Consider the following events for $u$: 
    $\mathcal{B}'_{1,i} := \{\abs{\inner{u}{\hat a_{r+i}}}\leq \sqrt{d}\theta/2\}$, 
    $\mathcal{B}'_{2,j} := \{\abs{\inner{u}{\hat a_j}} \geq 2\sqrt{d} k_l/m^3\}$, and 
    $\mathcal{B}'_{3,st} := \{\abs{\inner{u}{C_s \hat a_s - C_t \hat a_t}} \geq 2\sqrt{d}\alpha \}$. Set $\ev{E} = \cap_{i\in[k]} \ev{B}'_{1,i}$.
    With the concentration of $\norm{u}_2$ in $[\sqrt{d}/2,2\sqrt{d}]$, the target probability becomes:
    \begin{align}
        \label{equ: k+r2 bands 1}
        \prob[\cap_{i\in[k]}\mathcal{B}_{1,i} , 
        &\cap_{j\in[r]}{\mathcal{B}_{2,j}} , 
        \cap_{s,t\in[r],s\neq t} \mathcal{B}_{3,st} 
        ] 
        \geq \prob \left[ \ev{E} \setminus \Bigl( \right.
        (\{ \enorm{\proj_S u} \leq \sqrt{d}/2 \} \cap \ev{E})\nonumber \\
        &\bigcup \cup_{j \in [r]} ((\ev{B}'_{2,j})^c \cap \ev{E}) \bigcup (\{\enorm{u} \geq 2\sqrt{d}\} \cap \ev{E})
        \left.\bigcup  \cup_{s\neq t\in [r]}  ((\mathcal{B}'_{3,st})^c\cap \ev{E})\Bigr)\right] \nonumber\\
        &\geq p_1 - \sum_{s,t\in[r],s\neq t}\prob[\ev{E} , (\mathcal{B}'_{3,st})^c]. 
    \end{align}
    Now we consider the summand, which is the intersection of 
    $k+1$ type \RN{1} bands. 
    Take $s=1,t=2$ (the rest is similar) and write $v = C_1\hat a_1 - C_2\hat a_2 = \proj_Sv + \proj_{S^\perp}v$. 
    Then:
    \begin{equation}
        \label{equ: k+r2 bands 2}
        \begin{split}
            \prob[\ev{E}, (\mathcal{B}'_{3,12})^c] 
            &= \prob \bigl[\abs{\inner{u}{v}}\leq 2\sqrt{d}\alpha \bigm| \ev{E} \bigr] \prob[\ev{E}] \\
            &\leq \prob\bigl[\abs{\inner{u}{\proj_Sv}}\leq 2\sqrt{d}\alpha + \abs{\inner{u}{\proj_{S^\perp}v}} \bigm| \ev{E}\bigr] \prob[\ev{E}].
        \end{split}
    \end{equation}
    When conditioning on $\ev{E}$, $\inner{u}{\proj_{S^\perp}v}$ is bounded by:
    \begin{equation}
        \label{equ: k+r2 bands 3}
        \begin{split}
            \abs{\inner{u}{\proj_{S^\perp}v}} 
            &= \abs{u^\top \hat A_{>r}\hat A_{>r}^\dagger (C_1\hat a_1 - C_2\hat a_2)} 
            \leq \sqrt{dk}\theta \norm{\hat A_{>r}^\dagger(C_1\hat a_1 - C_2\hat a_2)}_2/2 \\ 
            &\leq \sqrt{dk}\theta\tau M k_l^{-1}m^3.
        \end{split}
    \end{equation}
    With \cref{equ: k+r2 bands 3}, we can drop the conditioning in \cref{equ: k+r2 bands 2}:
    \begin{equation}
        \label{equ: k+r2 bands 4}
        \begin{split}
            \prob[\ev{E}\cap (\mathcal{B}'_{3,12})^c] 
            &\leq \prob \bigl[\abs{\inner{u}{\proj_Sv}}\leq \alpha +  \sqrt{dk}\theta\tau M k_l^{-1}m^3 \bigr] \prob[\ev{E}] \\
            &\leq 2(\alpha +  \sqrt{dk}\theta\tau M k_l^{-1}m^3)/(\sqrt{2\pi}\norm{\proj_Sv}_2) \prob[\ev{E}] \\
            &\leq 2\tau M(\alpha +  \sqrt{dk}\theta\tau M k_l^{-1}m^3) \prob[\ev{E}].
        \end{split}
    \end{equation}
    The last inequality holds because the set $\{\hat a_1,\hat a_2,\hat a_{r+1},\dotsc,\hat a_{r+k}\}$ satisfies the robust Kruskal rank condition, and thus 
    \[
        \norm{\proj_Sv}_2= \norm{C_1\hat a_1 - C_2\hat a_2 - \hat A_{>r}\hat A_{>r}^\dagger v}_2\geq (\tau M)^{-1}\sqrt{C_1^2 + C_2^2 + \norm{\hat A_{>r}^\dagger v}_2^2 } \geq \sqrt{2}(\tau M)^{-1}.
    \]
    The combination of \cref{lem: uniform distribution lower bound 2,equ: k+r2 bands 1,equ: k+r2 bands 4} gives the desired probability.
\end{proof}
Finally, we are in a place to give the probability that all the events are true for $x,y$:
\begin{lemma}
    \label{lem: tensor probability bound 4}
    In the setting of \cref{lem: tensor probability bound 2}, 
    $\prob[\event_{1,x},\event_{1,y},\event_{2,x},\event_{2,y},\event_3] \geq p_1 p_2$.
    In particular, the choices $k_l = \sqrt{2\pi}\tau^{-1}M^{-1}m^3r^{-1}d^{-1/2}/64$, $\alpha = \tau^{-1}M^{-1}r^{-2}/16$ and 
    $\theta(r\sqrt{dk}\tau^2M^2 + 64r^3\tau^3M^3d\sqrt{k}/\sqrt{2\pi}) \leq 1/16$ 
    imply
    $\prob[\event_{1,x},\event_{1,y},\event_{2,x},\event_{2,y},\event_3] 
    \geq \bigl(\theta\sqrt{d}\bigr/8)^{2k}/256$.
\end{lemma}
\begin{proof}
The first part follows by combining \cref{lem: tensor probability bound 2,lem: tensor probability bound 3}.
For the second part, since $p_2 \leq p_1$, the claim follows by using our choices in $\prob[\event_{1,x},\event_{1,y},\event_{2,x},\event_{2,y},\event_3] \geq p_2^2$.
\details{
\[
(c_1\theta\sqrt{d})^k[1-r\sqrt{d}\tau M(4\hfrac{k_l}{m^3}+\sqrt{k}\tau M\theta) - r^2\tau M(\alpha +  \sqrt{dk}\theta\tau M k_l^{-1}m^3)] - 2e^{-c_2d}
   >0.    
\]
In other words, we need $r\sqrt{d}\tau M(4\hfrac{k_l}{m^3}+\sqrt{k}\tau M\theta) + r^2\tau M(\alpha +  \sqrt{dk}\theta\tau M k_l^{-1}m^3) < 1$.
Take $k_l = \tau^{-1}M^{-1}m^3r^{-1}d^{-1/2}/16$, $\alpha = \tau^{-1}M^{-1}r^{-2}/4$,then we only need: $(r\sqrt{dk}\tau^2M^2 + 16r^3\tau^3M^3d\sqrt{k})\theta< 1/2$.
}
\end{proof}

At this point we finished the analysis of the randomness in the first partial tensor decomposition, to recover the first $r$ components. 
In the next lemma we give the probability that random vectors $x',y'$ satisfy the assumptions of \cref{lem: deflation}. 
The events will be denoted by 
$\event'_{2,x} = \{\forall i\in[k], \abs{\inner{x'}{\hat a_{r+i}}} \geq k_l'/m^3\}$, 
$\event'_{2,y} = \{\forall i\in[k], \abs{\inner{y'}{\hat a_{r+i}}} \geq k_l'/m^3\}$ and 
$\event'_3 = \{\forall i\neq j, i,j\in[k], \abs{\inner{x'}{\hat a_{r+i}}/\inner{y'}{\hat a_{r+i}} - \inner{x'}{\hat a_{r+j}}/\inner{y'}{\hat a_{r+j}}} \geq \alpha' >0\}$.
\begin{lemma}
    \label{lem: tensor probability bound 5}
    Let $x',y'$ be iid.\ uniformly random in $\sphere^{d-1}$. For $\hat a_{r+1},\dotsc, \hat a_{r+k}$, and $k_l',\alpha'>0$, we have $\prob [\event'_{2,x},\event'_{2,y},\event'_3] \geq (1 -  k^2\sqrt{ed}\tau M\alpha' - \sqrt{ed}kk_l'/m^3)(1-\sqrt{ed}kk_l'/m^3)$. In particular, the choices $k'_l = m^3k^{-1}d^{-1/2}/(4\sqrt{e})$, $\alpha' = \tau^{-1}M^{-1}k^{-2}d^{-1/2}/(4\sqrt{e})$ imply  $\prob [\event'_{2,x},\event'_{2,y},\event'_3] \geq 3/8.$
\end{lemma}
\begin{proof}
    The first part reuses ideas from the proofs of \cref{lem: tensor probability bound 3,lem: tensor probability bound 2}.
    We first separate the intersection of events:
            $\prob[\event'_{2,x}\cap\event'_{2,y}\cap\event'_3] 
            = \prob[ \event'_{2,x} \cap \event'_3\giventhat \event'_{2,y}] P[\event'_{2,y}] 
            \geq (\prob[\event'_3 \mid \event'_{2,y}] - \prob[(\event'_{2,x})^c])\prob[\event'_{2,y}]$.
    By \cref{lem: uniform distribution lower bound}, $\prob[\event'_{2,x}]$ and $\prob[\event'_{2,y}]$ are at least $1- \sqrt{ed}kk_l'/m^3$. 
    Also
    \begin{equation*}
        \begin{split}
            \prob[(\event'_3)^c \mid \event'_{2,y}] &= \expectation\biggl[\prob\biggl[\min_{i\neq j, i,j\in[k]}\Bigabs{\frac{\inner{x'}{\hat a_{r+i}}}{\inner{y'}{\hat a_{r+i}}} - \frac{\inner{x'}{\hat a_{r+j}}}{\inner{y'}{\hat a_{r+j}}}} \leq \alpha' \biggm| y'\biggr]\biggm| \event'_{2,y}\biggr].
        \end{split}
    \end{equation*}
    Consider a uniform upper bound for $\prob[\min_{i\neq j, i,j\in[k]}\abs{\inner{x'}{C_i'\hat a_{r+i} - C_j'\hat a_{r+j}}} \leq\alpha']$, where $\abs{C_i'},\abs{C_j'}$ are lower bounded by $1$. 
    Therefore, again by \cref{lem: uniform distribution lower bound}, we have
        $\prob[(\event'_3)^c \vert \event'_{2,y}] \leq  k(k-1)\sqrt{ed}\tau M\alpha'/(2\sqrt{2}) \leq k^2\sqrt{ed}\tau M\alpha'$.
    Combining everything gives the desired result. The second part follows directly from our choices of $k_l'$ and $\alpha'$.
\end{proof}

\subsection{Putting everything together}
\label{sec: proof of main theorem}
In this subsection we prove \cref{thm:tensor decomposition main theorem}.
\begin{proof}[Proof of \cref{thm:tensor decomposition main theorem}]
Without loss of generality, assume $\pi$ is the identity, and assume for a moment that $\eps_{in},\theta$ are small enough so that: 
(1) the assumptions of \cref{lem: norm theorem,lem: norm theorem 2} are satisfied; and 
(2) $\eps_{\ref*{lem: first decomposition}}$ and $\eps_{\ref*{lem: deflation}}$ are smaller than 1 so that we can replace $\eps_{\ref*{lem: first decomposition}}^2$ and $\eps_{\ref*{lem: deflation}}^2$ by $\eps_{\ref*{lem: first decomposition}}$ and $\eps_{\ref*{lem: deflation}}$ in the expression of $\eps_{\ref*{lem: norm theorem}}$ and $\eps_{\ref*{lem: norm theorem 2}}$.
We trace the error propagation backwards and show how we can reach $\eps$ accuracy for the algorithm to terminate while achieving non-negligible success probability per iteration. 
The reconstruction error is bounded with \cref{lem: first decomposition,lem: norm theorem,lem: deflation,lem: norm theorem 2}:
\begin{equation}
    \label{equ: reconstruction error 1}
    \begin{split}
        \norm{T' - \tilde T}_F &\leq  \norm{\tilde T  - T}_F + \sum_{i\in[r+k]} \norm{a_i^{\otimes3} - \xi_i\tilde a_i^{\otimes3}}_F \\
        &\leq \eps_{in}+\sum_{i\in[r+k]} \Abs{\norm{a_i}_2^3 - s_i\xi_i}\norm{\tilde a_i^{\otimes3}}_F + \norm{\hat a_i^{\otimes3} - s_i^3\tilde a_i^{\otimes3}}_F\norm{a_i}_2^3 \\
        &\leq \eps_{in}+ 3rM^3\eps_{\ref*{lem: first decomposition}} + r\eps_{\ref*{lem: norm theorem}} + 3kM^3\eps_{\ref*{lem: deflation}} + k\eps_{\ref*{lem: norm theorem 2}}.
    \end{split}
\end{equation}
Collecting the results from \cref{lem: first decomposition,lem: norm theorem,lem: deflation,lem: norm theorem 2}, we have:
\begin{equation}
    \label{equ: error scale}
    \begin{array}{ll}
        \eps_{\ref*{lem: first decomposition}} =  O\bigl(\tau^4M^{10}k r^{5/2}k_l^{-2}\alpha^{-1}(\eps_{in} +\theta)\bigr) & 
        \eps_{\ref*{lem: deflation}} =  O(\tau^4M^7k^{5/2}r{k'_l}^{-2}{\alpha'}^{-1}\eps_{\ref*{lem: norm theorem}})\\
        \eps_{\ref*{lem: norm theorem}} = O(M^3m^3rk_l^{-1}\eps_{\ref*{lem: first decomposition}}) & 
        \eps_{\ref*{lem: norm theorem 2}}
        = O(M^3m^3k{k'_l}^{-1}\eps_{\ref*{lem: deflation}}).
    \end{array}
\end{equation}
With our choices of $k_l,\alpha,k_l',\alpha'$ in \cref{lem: tensor probability bound 4,lem: tensor probability bound 5}, \cref{equ: error scale} can be further written as:
\begin{equation*}
    \begin{array}{ll}
        \eps_{\ref*{lem: first decomposition}} = O\bigl(\tau^7M^{13}m^{-6}k r^{13/2}d(\eps_{in} + \theta)\bigr) &
        \eps_{\ref*{lem: deflation}} = O\bigl(\tau^{13}M^{25}m^{-12}k^{11/2}r^{19/2}d^{3}(\eps_{in} + \theta)\bigr)\\
        \eps_{\ref*{lem: norm theorem}} = O\bigl(\tau^8M^{17}m^{-6}k r^{17/2}d^{3/2}(\eps_{in} + \theta)\bigr) &
        \eps_{\ref*{lem: norm theorem 2}} = O\bigl(\tau^{13}M^{28}m^{-12}k^{13/2}r^{19/2}d^{7/2}(\eps_{in} + \theta)\bigr),
    \end{array}
\end{equation*}
which implies the reconstruction error is bounded by
\[\norm{T' - \tilde T}_F = O \bigl(\tau^{13}M^{28}m^{-12}k^{15/2}r^{19/2}d^{7/2}(\eps_{in} + \theta)\bigr).\]
This gives a polynomial $q(d,r,k,\tau,M,m^{-1}) = \Theta(\tau^{13}M^{28}m^{-12}k^{15/2}r^{19/2}d^{7/2})$, increasing in every argument, such that if we request that $\eps_{in} \leq \hfrac{\eps}{q(d,r,k,\tau,M,m^{-1})}$ and we set
\(
\theta = \hfrac{\eps}{q(d,r,k,\tau,M,m^{-1})},
\)
then $\norm{T' - \tilde T}_F \leq \eps$ 
(the first termination condition).
With this choice: 
(1) the assumptions of \cref{lem: tensor probability bound 4} are satisfied;
(2) for each iteration, with positive probability the events in \cref{lem: tensor probability bound 4,lem: tensor probability bound 5} happen; and 
(3) we can take $\eps_{\ref*{lem: first decomposition}} = \Theta(\tau^{-6}M^{-15}m^6r^{-3}d^{-5/2}\eps)$, $\eps_{\ref*{lem: deflation}} = \Theta(M^{-3}k^{-4}d^{-1/2}\eps)$ and they satisfy the assumptions of \cref{lem: norm theorem,lem: norm theorem 2}, respectively. 

Now we argue that the second termination condition, 
\(\max_{i\in[r+k]} \abs{\xi_i}^{1/3} \leq 2M\),
holds when the events in \cref{lem: tensor probability bound 4,lem: tensor probability bound 5} happen. 
Notice that at this point $\abs{\xi_i}$ is close to $\norm{a_i}_2^3$. 
Without loss of generality, $\max_{i\in[r+k]} \abs{\xi_i}^{1/3} = \abs{\xi_1}^{1/3}$.
Since $\forall x,y>0$, $\abs{y^{1/3} - x^{1/3}}\leq y^{-2/3}\abs{y-x}$, we have, $\forall i\in[r+k]$,
$    \Abs{\norm{a_i}_2 - \abs{\xi_i}^{1/3}} \leq \norm{a_i}_2^{-2}\Abs{\norm{a_i}_2^3 - \abs{\xi_i}}$,
which implies
$        \abs{\xi_1}^{1/3} \leq \norm{a_1}_2 + \Abs{\norm{a_1}_2 -\abs{\xi_1}^{1/3}} \leq \norm{a_1}_2 + \norm{a_1}^{-2}_2\varepsilon\leq M + m \leq 2M$,
where the second inequality comes from $\eps_{\ref*{lem: norm theorem}}\leq \eps$ and the third inequality comes from $\eps \leq \eps_{out} \leq m^3$.
Therefore, the algorithm terminates with a $2M$-bounded decomposition with reconstruction error at most $\eps$. 


Set
\( 
    \poly_{\ref*{thm:main tensor}}(d,\tau,M) = 2\poly_{\ref*{col: bhaskara corollary}}(2d,\tau,M,2M,d) \geq 2\poly_{\ref*{col: bhaskara corollary}}(r+k,\tau,M,2M,d)
\) and set
\(
    \poly'_{\ref*{thm:main tensor}}(d,\tau,M,m^{-1}) = q(d,d,d,\tau,M,m^{-1}) \poly_{\ref*{thm:main tensor}} \geq q(d,r,k,\tau,M,m^{-1}) \poly_{\ref*{thm:main tensor}}
\).\footnote{Recall that $k\leq r \leq d$ by assumption.}
When the algorithm terminates, we have 
$\norm{T - T'}_F 
\leq \eps + \eps_{in} 
\leq \eps + \frac{\eps}{q} 
\leq \frac{\eps_{out}}{\poly_{\ref*{thm:main tensor}}} + \frac{\eps_{out}}{q\poly_{\ref*{thm:main tensor}}} 
\leq \frac{\eps_{out}}{\poly_{\ref*{col: bhaskara corollary}}(r+k,\tau,M,2M,d)}$.
Thus, we can apply \cref{col: bhaskara corollary} and obtain component-wise $\eps_{out}$ accuracy.

For the running time, in each iteration, steps \ref{line:decomposition} and \ref{line:decomposition2} run in time $\poly(d, \eps^{-1}, \tau, M, m^{-1})$.
Least squares steps \ref{line:leastsquares} and \ref{line:leastsquares2} and the rest take $\poly(d)$ time.
By \cref{lem: tensor probability bound 4,lem: tensor probability bound 5}, the success probability per iteration is at least
$ 3\bigl(\theta\sqrt{d}/8\bigr)^{2k}/2^{11}$,
which implies that the expected number of iterations is at most $2^{11}(\theta\sqrt{d}/8)^{-2k}/3$ and the expected running time is at most $\poly(d^k,\eps^{-k},\tau^k,M^k, m^{-k})$. 
Since $\eps = \eps_{out}/\poly_{\ref*{thm:main tensor}}$, the expected running time is also at most $\poly(d^k,1/\eps_{out}^k,\tau^k,M^k, m^{-k})$.
\end{proof}


\appendix

\section{Estimating cumulants}
    \label{sec: sample cumulants}
    In this section we provide technical details about the unbiased estimators of cumulants, called $k$-statistics. 
    They are the unbiased estimator for cumulants with the minimum variance, and are long studied in the statistics community. 
    We provide the formula for the 3rd $k$-statistic given in \cite[Chapter 4]{mccullagh2018tensor} here:
    \begin{fact}
        \label{fact: sample cumulant}
        Given iid.\ samples $x_1,\ldots,x_N$ of random vector $X$, the $k$-statistic for the 3rd cumulant of $X$ is:
        $
            k_3(r,s,t) = \frac{1}{N}\sum_{i,j,k \in[N]}\phi^{(ijk)}(x_i)_r(x_j)_s(x_k)_t
        $,
        where $r,s,t$ are the position indices in the tensor, and $\phi^{(ijk)}$ is the coefficient given by: it is invariant under permutation of indices, and for distinct $i,j,k \in[N]$:
        \begin{equation}
            \label{cumulant coefficients}
            \phi^{(iii)} =\frac{1}{N},\quad \phi^{(iij)} = -\frac{1}{N-1},\quad \phi^{(ijk)} = \frac{2}{(N-1)(N-2)}.
        \end{equation}
    \end{fact}

    To obtain the entry-wise concentration bound for $k_3$, we begin by bounding the variance of each entry in $k_3$:
    \begin{lemma}
        \label{lem: variance of k3}
        Let $X$ follow a distribution as in \cref{equ: mixture}. The 3rd $k$-statistics $k_3$ of $X$ satisfies:
        \(
            \var \bigl( k_3(r,s,t) \bigr) = O(\hfrac{\max_{t\in [d]}\expectation[X_t^6]}{N})
        \).
    \end{lemma}
    \begin{proof}
        An essentially identical result for the 4th cumulant is show in \cite[Lemma 4]{DBLP:conf/colt/BelkinRV13}. The argument here is the same. We provide a proof in the supplementary materials.
        \end{proof}
        Using Chebyshev's inequality yields the follow sample bound immediately:
        \begin{lemma}
            \label{lem: cumulant sample complexity}
            Given $\epsilon,\delta\in(0,1)$, the entry-wise error between $k_3$ and $K_3(X)$ is at most $\epsilon$ with probability at least $1-\delta$ when using 
            \( 
                N \geq \Omega\left( \epsilon^{-2}\delta^{-1} \max_{t\in [d]}\expectation[X_t^6] \right)
            \)
            samples.
        \end{lemma}
        
\section{Technical lemmas}

    \subsection{Perturbed SVD bounds}
        We state Wedin's theorem, a ``$\sin(\theta)$ theorem'' for perturbed singular vectors as well as Weyl's inequality for SVD. 
        The following results are from \cite{stewart1998perturbation}.
        \begin{theorem}[Weyl's inequality]
            \label{thm: weylsvd}
            Let $A, E\in \Real^{d_1\times d_2}$ with $d_1\geq d_2$. Denote the singular values in non-increasing order of $A$ and $A+E$ by $\sigma_i$ and  $\Tilde{\sigma}_i$, respectively. 
            Then 
            \(
                \abs{\sigma_i-\Tilde{\sigma}_i}\leq \norm{E}_2
            \).
        \end{theorem}
        \begin{theorem}[Wedin]
            \label{thm: wedin}
            With the notation from \cref{thm: weylsvd}, let a singular value decomposition of $A$ be:
            $    [U_1, U_2, U_3]^\top A[V_1, V_2] = \begin{bmatrix} \Sigma_1 & 0 \\ 0 & \Sigma_2 \\ 0 & 0 \end{bmatrix}$,
            where the singular values can be in arbitrary order. Let the perturbed version be:
            $    [\Tilde{U}_1, \Tilde{U}_2, \Tilde{U}_3]^\top (A+E)[\Tilde{V}_1, \Tilde{V}_2] = \begin{bmatrix} \Tilde{\Sigma}_1 & 0 \\ 0 & \Tilde{\Sigma}_2 \\ 0 & 0\end{bmatrix}$.
            Let $\Phi$ be the matrix of canonical angles between the column spaces of $U_1$ and $\tilde U_1$, and $\Theta$ be that of $V_1$ and $\tilde V_1$, respectively. 
            Let $\delta = \min\{ \min_i\Tilde{\Sigma}_{1,ii}, \min_{i,j} \abs{\Tilde{\Sigma}_{1,ii} - \Sigma_{2,jj}}\}$. 
            Then
            \(
                \sqrt{\norm{\sin\Phi}^2_2 + \norm{\sin\Theta}^2_2}\leq \hfrac{\sqrt{2}\norm{E}_2}{\delta}.
            \)
        \end{theorem}
    \subsection{Probability tail bounds}
    \begin{lemma}[\cite{dasgupta2003elementary,hsu2013learning}]
        \label{lem: uniform distribution lower bound}
        Suppose $\delta\in(0,1)$, $M\in \Real^{d\times d}$, $Q$ is a finite subset of $\Real^d$ and 
        $X$ is a uniformly random vector in $\sphere^{d-1}$.
        Then
        \( 
            \prob\left[\min_{q\in Q} \abs{\inner{X}{Mq}} \geq \frac{\delta\min_{q\in Q} \norm{Mq}_2}{\sqrt{ed}\abs{Q}} \right]\geq 1-\delta
        \).
    \end{lemma}
    \begin{lemma}
        \label{lem: uniform distribution lower bound 2}
        Let $X\in\Real^d$ be a standard Gaussian random vector, $a_1, \dotsc, a_k \in \sphere^{d-1}$, and $t \in [0,1]$.
        Then
        \(
            \prob \bigl[(\forall i) \abs{\inner{X}{a_i}} \leq t \bigl] \geq (t/4)^k
        \).
    \end{lemma}
    \begin{proof}
        The claim follows immediately from the Gaussian correlation inequality and the fact that the one-dimensional standard Gaussian density in $[-1,1]$ is at least $(2\pi e)^{-1/2} \geq 1/8$.
    \end{proof}

\section*{Acknowledgments}
We would like to thank Nina Amenta, Jes\'us De Loera, Shuyang Ling, Naoki Saito and James Sharpnack for helpful discussions. 

\bibliographystyle{abbrv}
\bibliography{ref}

\end{document}